\documentclass[10pt]{article}
\usepackage[margin=1in]{geometry}
\usepackage[protrusion=true,final]{microtype}

\usepackage{lmodern}

\usepackage{enumitem}

\usepackage[dvipsnames]{xcolor}
\usepackage[round]{natbib}
\definecolor{lava}{rgb}{0.81, 0.06, 0.13}
\usepackage{hyperref}
\usepackage[hyperpageref]{backref}
\usepackage{url}
\hypersetup{
    colorlinks=true,
    linkcolor=lava,
    citecolor=blue,
    urlcolor=black,
}

\usepackage{amsmath}
\usepackage{amsthm}
\usepackage{amsfonts}
\usepackage{amssymb}
\usepackage{mathtools}
\usepackage{physics}
\usepackage{physconst}
\usepackage{units}
\usepackage{textcomp, gensymb}
\usepackage{mathrsfs}
\usepackage{bm} 
\usepackage{xfrac}

\usepackage{graphicx} 
\usepackage{tabulary} 
\usepackage{booktabs}
\usepackage{colortbl}
\usepackage{subfig}
\usepackage{tikz}
\usepackage{pgfplots}
\pgfplotsset{compat=1.18}

\newcommand{\NN}{\mathbb{N}}

\newcommand{\RR}{\mathbb{R}}

\newcommand{\calN}{\mathcal{N}}

\newcommand{\calX}{\mathcal{X}}

\newcommand{\EE}[1]{\mathbb{E}\left[#1\right]}
\newcommand{\dEE}[2]{\mathbb{E}_{#1}\left[#2\right]}

\newcommand{\udrEE}[2]{\underset{#1}{\mathbb{E}}\left[#2\right]}

\newcommand{\Var}[1]{\mathrm{Var}\left[#1\right]}

\renewcommand{\norm}[1]{\left\lVert#1\right\rVert}
\newcommand{\inner}[2]{\left\langle #1, #2 \right\rangle}
\newcommand{\ds}{\displaystyle}

\DeclareMathOperator*{\argmin}{arg\,min}

\usepackage{listings}
\usepackage{algorithm}
\usepackage{algorithmicx}
\usepackage[noend]{algpseudocode}
\usepackage{cases}

\newtheorem{theorem}{Theorem}[section]

\newtheorem{lemma}[theorem]{Lemma}
\newtheorem{proposition}[theorem]{Proposition}

\newtheorem{corollary}{Corollary}
\newtheorem{claim}[theorem]{Claim}

\theoremstyle{definition}
\newtheorem{definition}[theorem]{Definition}
\newtheorem{assumption}[theorem]{Assumption}

\newtheorem{example}[theorem]{Example}

\theoremstyle{remark}

\newtheorem*{remark}{Remark}

\title{Incentivizing Truthful Collaboration in Heterogeneous Federated Learning}

\newcommand{\ttic}{Toyota Technological Institute at Chicago}
\newcommand{\insait}{INSAIT, Sofia University ``St. Kliment Ohridski''}

\author{
    Dimitar Chakarov%
    \thanks{\ttic, \texttt{chakarov@ttic.edu}. Part of this work was done while DC was visiting INSAIT, Sofia University ``St. Kliment Ohridski.''}
    \and
    Nikita Tsoy%
    \thanks{\insait, \texttt{nikita.tsoy@insait.ai}.}
    \and
    Kristian Minchev%
    \thanks{\insait, \texttt{kristian.minchev@insait.ai}.}
    \and
    Nikola Konstantinov%
    \thanks{\insait; \texttt{nikola.konstantinov@insait.ai}.}
}

\date{}

\newcommand{\blfootnote}[1]{%
    \begingroup
    \renewcommand\thefootnote{}\footnote{#1}%
    \addtocounter{footnote}{-1}%
    \endgroup
}

\allowdisplaybreaks

\renewcommand{\paragraph}[2][1em]{\vspace{#1}\noindent\textbf{#2}\ }

\begin{document}

\maketitle

\begin{abstract}
    Federated learning (FL) is a distributed collaborative learning method, where multiple clients learn together by sharing gradient updates instead of raw data. However, it is well-known that FL is vulnerable to manipulated updates from clients. In this work we study the impact of data heterogeneity on clients' incentives to manipulate their updates. First, we present heterogeneous collaborative learning scenarios where a client can modify their updates to be better off, and show that these manipulations can lead to diminishing model performance. To prevent such modifications, we formulate a game in which clients may misreport their gradient updates in order to ``steer'' the server model to their advantage. We develop a payment rule that provably disincentivizes sending modified updates under the FedSGD protocol. We derive explicit bounds on the clients' payments and the convergence rate of the global model, which allows us to study the trade-off between heterogeneity, payments and convergence. Finally, we provide an experimental evaluation of the effectiveness of our payment rule in the FedSGD, median-based aggregation FedSGD and FedAvg protocols on three tasks in computer vision and natural language processing. In all cases we find that our scheme successfully disincentivizes modifications.
\end{abstract}

\blfootnote{A version of this work appeared in the Optimization for Machine Learning Workshop (OPT-2024) at the Thirty-eight Conference on Neural Information Processing Systems, NeurIPS 2024.}

\section{Introduction}

Federated learning (FL) \citep{mcmahan2017communication} enables the efficient training of machine learning models on large datasets, distributed among multiple stakeholders, via gradient updates shared with a central server. FL has the potential to provide state-of-the-art models in multiple domains where high-quality training data is scarce and distributed~\citep{kairouz2021advances}. For example, FL can be a major advancement in fields like healthcare~\citep{rieke2020future}, agriculture~\citep{durrant2022role}, transportation~\citep{tan2020federated} and finance~\citep{long2020federated, oualid2023federated}.

Unfortunately, the distributed nature of standard FL protocols makes them susceptible to clients misreporting their gradient updates. Indeed, it is known that the presence of market competition \citep{dorner2024incentivizing}, privacy concerns \citep{tsoy2024provable} and high data gathering costs \citep{karimireddy2022mechanisms} may incentivize clients to communicate updates that are harmful for the global model. Furthermore, prior work has shown that a small fraction of malicious participants can damage the learned model with seemingly benign updates \citep{blanchard2017machine,alistarh2018byzantine}. These issues bring the practical merit of federated learning in the presence of misaligned incentives into question.

In the current work, we show that incentives for update manipulation may appear even between clients who are solely interested in their own accuracy, as long as they have different data distributions. Since data heterogeneity is ubiquitous in common federated learning scenarios \citep{mcmahan2017communication,kairouz2021advances}, this implies that clients could be incentivized to manipulate their updates in realistic scenarios and even without the presence of explicitly conflicting goals, such as those arising from competition and privacy.

\subsection{Main contributions}

We introduce a game-theoretic framework that models rational client behavior in the context of heterogeneous FL and provide a payment mechanism that successfully disincentives possibly harmful gradient modifications, which ensures training convergence. Our main theoretical contributions are as follows:
\begin{itemize}[noitemsep, topsep=0pt, parsep=0pt, partopsep=0pt]
    \item We propose a game-theoretic framework for studying the impact of heterogeneous data on the incentive of clients to modify their gradient updates before sending them to the server. This model allows for several natural types of update manipulations, as well as various reward functions for the clients. We show that this enables modeling various real-world client incentives.
    \item We design a budget-balanced payment mechanism that when used in conjunction with the FedSGD protocol achieves both (1) \emph{$\varepsilon$-approximate incentive compatibility}---truthful reporting results in utility that is $\varepsilon$-close to optimal for a client $i$ given that everyone else is sending truthful updates, and (2) \emph{$\varepsilon$-approximately truthful reporting}---the best response of client $i$ when everyone else is sending truthful updates is $\varepsilon$-close to the truthfulness.
    \item We prove the convergence of FedSGD with standard rates for smooth, strongly-convex objectives, when executed for clients who report $\varepsilon$-approximately truthfully under our payment scheme. In the same context, we also provide an upper bound on the total payment of each client.
\end{itemize}
Beyond theory, we empirically investigate the effectiveness of our payment mechanism on three (non-convex) tasks: image classification on the FeMNIST dataset, sentiment analysis on the Twitter dataset, and next-character prediction on the Shakespeare dataset~\citep{caldas2018leaf}. We observe that our payment mechanism successfully eliminates the incentive for gradient manipulation in all three tasks. Finally, we provide empirical evidence that our method generalizes to other FL protocols, such as median-based FedSGD, which aims to aggregate gradients robustly, and FedAvg, the classical FL protocol of~\citet{mcmahan2017communication}.

\subsection{Related work}

\paragraph{Heterogeneity in Federated Learning.}
The impact of data heterogeneity on the quality of the learned model is of central interest in the FL literature. Some works study how to train a central model and also provide personalization, in order to maximize the accuracy for each client \citep{mcmahan2017communication,mansour2020three,marfoq2022personalized,mishchenko2023partially}. Others focus on the impact of heterogeneity on model convergence and on designing algorithms that provide more accurate centralized models in a heterogeneous environment~\citep{karimireddy_scaffold_2020, khaled2020tighter, koloskova_unified_2020, woodworth_minibatch_2020, patel2024limits}.

These works tackle issues of data heterogeneity from the server's perspective, while we take a mechanism-design point of view and focus on the clients' behavior. Prior work~\citep{chayti2021linear} has also explored how to adapt to heterogeneity from a client's perspective, by optimally weighting local and server updates. Instead, in the current paper, we study how clients might manipulate their messages to the server, thus damaging the global training process, and our aim is to mitigate this potential manipulation. Another client behavior driven by the heterogeneity of client data is studied by \citet{donahue2020model, donahue2021optimality}, who, however, model FL as a coalitional game, where players need to decide how to cluster in groups to improve their performance. In contrast, we study non-cooperative games and clients' incentives to manipulate their updates.

\paragraph{Robustness in Federated Learning.}
Prior work has explored the robustness of FL to noise and bias towards subgroups, e.g.~\citet{abay2020mitigating, fang2022robust}. Several works also consider Byzantine-robust learning~\citep{blanchard2017machine, yin2018byzantine, alistarh2018byzantine}. We refer to \citet{shejwalkar2022back} for a recent survey. Moreover, ~\citet{pillutla2022robust} propose a robust aggregation procedure that is similar to the way we construct our penalty scheme.

Our work takes a different approach towards securing FL from potentially harmful updates: we model the clients as rational agents in a game and seek to ensure that honest reporting is close to their utility optimal. This way the global model can benefit from diverse data, while disinsentivizing gradient manipulation.

\paragraph{Incentives in Collaborative Learning.}
A major research direction is that of studying whether clients have an incentive to join the FL protocol, relative to participation costs (e.g. compute resources and data collection costs). We refer to \citet{tu2021incentive, zhan_survey_2022} for recent surveys. Additionally, \citet{tsoy2024provable} explore the impact of privacy concerns on FL participation incentives. Among works that consider incentives for manipulating updates, \citet{karimireddy2022mechanisms} focus on free-riding,~\citet{han2023effect} study defection,~\citet{blum_one_2021} look at agents who wish to keep their data collection low, while~\citet{huang_evaluating_2023} try to incentivize diverse data contributions. Moreover, \citet{dorner2024incentivizing} consider manipulating updates due to incentives stemming from competition, focusing on homogeneous SGD.

In contrast to these works, we consider clients who are solely interested in their own accuracy and utility, but, as we show, may still have conflicting incentives due to data heterogeneity.

\section{Preliminaries}

In this section we introduce our heterogeneous FL setup.

\subsection{FL setup and protocol}\label{subsection:fl-setup}

\paragraph{Learning setup.}
We consider a standard FL setting with $N$ clients who seek to obtain an accurate model by exchanging messages (gradient updates) with a central server, which orchestrates the protocol. All clients work with a shared loss function $f(\theta; z)$ that is differentiable, convex and smooth in $\theta \in \mathcal{Z}$ for every $z$. Each client $i$ has her own distribution $D_i$ over the data $z \in \mathcal{Z}$,\footnote{We do not assume any parametric form of the distributions.} and is interested in minimizing the expected loss with respect to their own distribution, i.e. the loss function of client $i$ is $F_i(\theta) = \dEE{z \sim D_i}{f(\theta; z)}.$ We assume that
\[
    \nabla F_i(\theta) = \dEE{z \sim D_i}{f(\theta; z)},
\]
and that there exists a scalar $\sigma \geq 0$, such that
\[
    \udrEE{z \sim D_i}{\norm{f(\theta; z) - \nabla F_i(\theta)}^2} \leq \sigma^2.
\]
In other words, we assume that for each client their stochastic gradient is an unbiased estimator for their full gradient and has bounded variance. Unless otherwise stated, all norms are $\ell_2$ norms. We further require that $F_i$ is $m$-strongly-convex and $H$-smooth for each client $i$.
\begin{definition}[Strong convexity]\label{definition:strong-convexity-main}
    A differentiable function $F : \calX \subseteq \RR^d \to \RR$ is \emph{$m$-strongly-convex} if for all $x, y \in \calX$:
    \[
        F(x) \geq F(y) + \inner{\nabla F(y)}{x - y} + \frac{m}{2} \norm{x - y}^2.
    \]
\end{definition}

\begin{definition}[Smoothness]
    A differentiable function $F : \calX \subseteq \RR^d \to \RR$ is \emph{$H$-smooth} if for all $x, y \in \calX$:
    \[
        \norm{\nabla F(x) - \nabla F(y)} \leq H \norm{x - y}.
    \]
\end{definition}
\noindent Finally, the server's objective is to minimize the average expected loss of all clients, i.e. \[F(\theta) = \frac{1}{N} \sum_{i = 1}^N F_i(\theta).\]
\vspace{-1em}

\paragraph{FL protocol.}
We consider the standard FedSGD protocol, where the server asks the clients to send stochastic gradients at the current model, with respect to their own data distribution. At each time step $t$ client~$i$ computes a stochastic gradient $g_i(\theta_t) \coloneqq \nabla f(\theta_t; z)$ by sampling from their distribution $z \sim D_i$ independently of everything else. We also let $e_i(\theta_t) = g_i(\theta_t) - \nabla F_i(\theta_t)$ be the gradient noise. The server then updates the central model by averaging the updates and taking an SGD step, i.e. $\theta_{t+1} = \theta_t - \gamma_t \frac{1}{N}\sum_{i=1}^N g_i(\theta_t),$ where $\gamma_t$ is the learning rate at step $t$.

\subsection{Heterogeneity assumptions}

To describe the relation between the data distributions, we invoke an assumption reminiscent to the standard bounded first-order heterogeneity assumption from the FL literature on the convergence rates of local and mini-batch SGD~\citep{karimireddy_scaffold_2020, khaled2020tighter, koloskova_unified_2020, woodworth_minibatch_2020, patel2024limits}. Assumption~\ref{assumption:gradient-difference} restricts the size of the gradient $\nabla F_i$ of client~$i$'s objective $F_i$ relative to the gradient $\nabla F$ of the aggregate objective $F$.  While standard first-order assumptions previously used in~\cite{karimireddy_scaffold_2020, khaled2020tighter, koloskova_unified_2020, woodworth_minibatch_2020, patel2024limits} (for formal statements see Assumption~\ref{assumption:bounded-first-order-heterogeneity} and~\ref{assumption:relaxed-first-order-heterogeneity} in Appendix~\ref{appendix:heterogeneity-assumptions}) usually require the gradients of the objectives to be close in some vector norm, we only require them to be close in magnitude.

\begin{assumption}[Bounded Gradient Difference]\label{assumption:gradient-difference}
    For every client $i$ and every $\theta \in \RR^d$, we have \[\abs{\norm{\nabla F_i(\theta)}^2 - \norm{\nabla F(\theta)}^2} \leq \zeta^2.\] As a consequence $\abs{\norm{\nabla F_i(\theta)}^2 - \norm{\nabla F_j(\theta)}^2} \leq 2\zeta^2.$
\end{assumption}

Next, Assumption~\ref{assumption:variance-difference} controls the difference between the variance of stochastic gradients. This allows us to cover scenarios, where the objectives are sufficiently similar, but the variance of some client~$i$ (relative to others) might induce an incentive to misreport. We employ it in our bound on individual payments in Proposition~\ref{theorem:bound-on-payments}.

\begin{assumption}[Bounded Variance Difference]\label{assumption:variance-difference}
    For any pair of clients $i, j \in [N]$ and any $\theta \in \RR^d$, we have \[\abs{\EE{\norm{e_i(\theta)}^2} - \EE{\norm{e_j(\theta)}^2}} \leq \rho^2,\]
    where $\EE{\norm{e_i(\theta)}^2}$ is the variance of the $\ell_2$ norm of the error vector $e_i(\theta) = g_i(\theta) - \nabla F_i(\theta)$.
\end{assumption}

We note that the parameters $\zeta$ and $\rho$ have a natural interpretation as the amount of heterogeneity among the FL clients.

\section{Motivating examples}\label{section:motivating-examples}

We now present a couple of motivating examples about how a strategic client can take advantage within the protocol by misreporting, due to heterogeneity in the clients' data. First, we present a conceptually easier example with single-round mean estimation. This example provides intuition about why upscaling a gradient can help steer the aggregate gradient closer to the true gradient of some client $i$, and in the process move the model to a position that is more favorable for client $i$. Note that mean estimation can be interpreted as a single step of the learning process. Then we give a simplified example with gradient descent and illustrate a scenario where one client's misreporting can make other clients worse off. 

\subsection{Example with Mean Estimation}

Suppose $N$ clients seek to estimate their respective means $\mu_1,\ldots,\mu_{N} \in \mathbb{R}$, where without loss of generality $\mu_1>\mu_2>\ldots>\mu_{N}$. First, each client gets a sample $x_{i} \sim \mathcal{N}(\mu_i, \sigma^2)$, and sends a message $m_{i}$ (supposedly their sample $x_{i}$) to the server. Then the server computes an aggregate $\overline{\mu} = \frac{1}{N} \sum_{i=1}^{N} m_{i}$ and broadcasts it to all clients. Each client $i$ would like to receive an estimate of their local mean with minimal mean squared error $\mathbb{E}[ \left( \overline{\mu} - \mu_i \right)^2 ]$.

\begin{proposition} \label{proposition:mean_example}
   Let $\mu = \frac{1}{N}\sum_{i=1}^N \mu_i$ and assume that $\mu_1(\mu_1 - \mu) > \sigma^2 / N$ and that everyone but the first client truthfully reports their sample. If client $1$ truthfully reports their sample $x_1$, then $\mathbb{E}\left[ \overline{\mu} \right] = \mu$ and $Var(\overline{\mu}) = \sigma^2 / N$, and so \[\mathbb{E}\left[ \left( \overline{\mu} - \mu_1 \right) ^2 \right] = \left( \mu - \mu_1 \right) ^2 + \frac{\sigma^2}{N}.\]
    However, there exists a constant $c>1$, such that if client $1$ sends $c x_{1}$, they reduce their MSE to \[\mathbb{E}\left[ (\bar{\mu} - \mu_1)^2 \right] = (\mu-\mu_1)^2 + \frac{\sigma^2}{N} -  \frac{(\sigma^2/N - \mu_1(\mu_1 - \mu))^2}{\mu_1^2+\sigma^2}.\]
\end{proposition}

The proof of Proposition~\ref{proposition:mean_example} alongside additional results for this example setup can be found in Appendix~\ref{appendix:mean-estimation}. The assumption $\mu_1(\mu_1 - \mu) > \sigma^2 / N$ holds whenever the heterogeneity of the data distributions ($\mu - \mu_1$) is large with respect to the variance $\sigma^2/N$ of the aggregated message. 

The conclusion of Proposition~\ref{proposition:mean_example} is that there is some client that is incentivized to magnify their sample in order to reduce their final error. In words, client $1$, whose target mean $\mu_1$ is far from the global $\mu$, can attempt to bias the aggregated mean estimate in their favor by altering their update to be more extreme.

\begin{remark}
    In Appendix~\ref{appendix:game-eqilibrium} we investigate the Nash equilibrium of the mean estimation game, where all clients can scale their gradient updates by a personal constant $c_i$. Contrast this with Proposition~\ref{proposition:mean_example}, which considers the incentive-compatible regime---we fix all but client~$i$ to report truthful gradients and then derive the optimal strategy for~$i$. We find that, for a large number of clients and at the equilibrium, the clients have an incentive to deviate and, on average, increase their weight, thus making their contribution a biased estimate of their true mean. This results in a higher error for all participants, because the final estimate is now biased and has higher variance.
\end{remark}

\begin{figure}[t]
    \centering
    \begin{tikzpicture}
    \begin{axis}[
        xmin=-4,
        xmax=8,
        ymin=-1,
        ymax=10,
        samples=100,
        xtick={-6,-4,-2, 0, 2, 4, 6, 8},
        ytick={0, 2, 4, 6, 8},
        legend style={
            at={(1,1)}, 
            anchor=north east, 
            font=\scriptsize,
            row sep=-2pt,
            inner xsep=1pt,
            inner ysep=1pt,
        },
        legend cell align={left},
        legend image post style={scale=0.5},
        grid=major,
        width=9cm,
        height=8cm,
    ]
    
    \addplot[color=red, ultra thick] {x^2};
    \addlegendentry{$x^2$}

    \addplot[color=red, ultra thick, dotted, line width=0.75mm] {3 * x^2};
    \addlegendentry{${\mathbf{3}} x^2$}
    
    \addplot[color=blue, ultra thick] {(x - 1)^2 + 1};
    \addlegendentry{$(x - 1)^2 + 1$}

    \addplot[color=green, ultra thick] {3 * (x - 2)^2 + 2};
    \addlegendentry{$3 (x - 2)^2 + 2$}

    \addplot[color=orange, ultra thick] {(x - 0.5)^2 + 3};
    \addlegendentry{$(x - 0.5)^2 + 3$}

    \addplot[color=black, ultra thick] {1.5 * x^2 - 3.75 * x + 4.8125};
    \addlegendentry{$x^* = 5/4$}

    \addplot[color=black, ultra thick, dotted, line width=0.75mm] {2 * x^2 - 3.75 * x + 4.8125};
    \addlegendentry{$x^{**} = 15/16$}

    \end{axis}
    \end{tikzpicture}
    \caption{The plot follows Proposition~\ref{proposition:sgd-example}. Clients are represented by colors, and their respective loss functions (variants of quadratic loss) are shown next to their color in the legend. The black curve is the average loss function over all clients, and the legend shows the optimum. The red client scales their gradients by a constant, in this case $3 \times$; the dotted red line is their new loss function, and the dotted black line is the new average loss. The new global optimum is better for the red client, while it's worse for the green client.}
    \label{figure:sgd-example}
\end{figure}
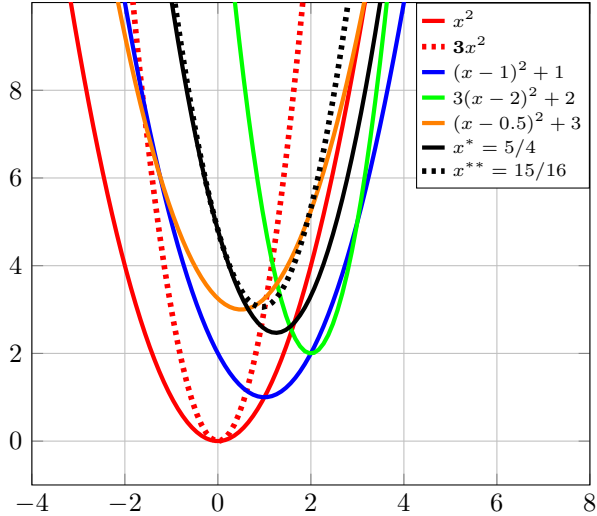

\subsection{Example with Stochastic Gradient Descent}
In order to show the effects of gradient modification we opt to work with the simpler gradient descent, so both the objective function and the gradients are deterministic, without stochasticity from data sampling or noise. Proposition~\ref{proposition:sgd-example} considers a set of four clients each having a square loss, where one of them has incentive to amplify their gradient update.
\begin{example}\label{proposition:sgd-example}
    Consider four clients with objective functions: $F_1(x) = x^2$, $F_2(x) = (x - 1)^2 + 1$, $F_3(x) = 3 (x - 2)^2 + 2$, $F_4(x) = (x - 0.5)^2 + 3$. Their average objective is $F(x) = 1.5 x^2 - 3.75 x + 4.8125$ and achieves minimal value at $x^* = \frac{5}{4}$. Suppose client 1 sends a gradient $g(x) = 6x$ as if his objective is $F'_1(x) = 3x^2$. The new average objective is $F'(x) = 2 x^2 - 3.75 x + 4.8125$ with minimal value at $x^{**} = \frac{15}{16}$. Then client 1 is better off, i.e. $F_1(x^{**}) < F_1(x^*)$, while client 3 is worse off, i.e. $F_3(x^{**}) > F_3(x^*)$. 
\end{example}
See Figure~\ref{figure:sgd-example} for visualization of Proposition~\ref{proposition:sgd-example}. The proposition illustrates how gradient amplification can be beneficial for some client, while making other clients worse off. The final model is better performing for client 1 (red on Figure~\ref{figure:sgd-example}), while client 1's modification leads to client 3 (green on Figure~\ref{figure:sgd-example}) being worse off compared to when all clients were truthful.

\section{Game-theoretic framework}\label{section:game}

Motivated by the examples in Section~\ref{section:motivating-examples} we introduce a game setup which captures these strategic interactions.

\paragraph{Action Space.}
We consider a game in which the clients seek to improve their own loss function by manipulating the messages they send to the server. Specifically, at time $t$ each client sends message $m^i_t = a^i_t g_i(\theta_t) + b^i_t \xi^i_t$, for some constants $\abs{a^i_t} \geq 1$, $b^i_t \geq 0$ chosen by the client. Here $\xi^i_t$ is a random vector that is independent of everything else and satisfies $\EE{\xi^i_t} = 0$ and $\EE{\norm{\xi^i_t}^2} = 1$. We also cover generalizations of the action space in Subsection~\ref{subsection:action-spaces-extension}. The rest of the process remains unchanged, with the server computing $\bar{m}_t = \frac{1}{N} \sum_{i = 1}^N m^i_t$ and $\theta_{t + 1} = \theta_t - \gamma_t \bar{m}_t$, and communicating $\theta_{t + 1}$ to all clients at each round.

The scaling factor $a^i_t$ gives a client~$i$ the ability to magnify their gradient update, as in the examples in Section \ref{section:motivating-examples}, so that the aggregation method computes a mean gradient that could produce a final model that is more beneficial for client~$i$. The additional noise term $b^i_t \xi^i_t$ is a natural way to also account for (1) clients wanting to obscure portions of their data, while amplifying the gradient effect of other portions, and/or (2) the clients adding differential-privacy noise to their updates for privacy-preserving reasons.

\paragraph{Utility and Reward Functions.}
As is standard in game theory, we assume that the clients are rational, i.e. they seek to select actions $(a^i_t, b^i_t)$ that increase their utility. We let the utility $U_i$ of client $i$ to be \[U_i = R_i(\theta) - p_i,\] where $R_i$ is the reward client $i$ gets from the end model $\theta$, and $p_i$ is the total payment client $i$ pays. Under a standard FL setting, $p_i = 0$ and clients observe only a reward that is a function of the final model. We introduce the option of payments to mitigate the incentive for gradient modification. In our theoretical analysis (see Section~\ref{section:theoretical-results}) we require that $R_i(\theta)$ is $L$-Lipschitz in $\theta$.

\begin{definition}[Lipschitzness]
    A function $R : \calX \subseteq \RR^n \to \RR^m$ is \emph{$L$-Lipschitz} if for all $x, y \in \calX$:
    \[
        \norm{R(x) - R(y)} \leq L \norm{x - y}.
    \]
\end{definition}

\paragraph{Desiderata.}
We seek to design a payment mechanism, or payment rule, such that the clients are incentivized to send meaningful updates, in order for the final model to achieve small loss on the global objective $F(\theta)$. In particular, we would like our payment mechanism to satisfy two properties:
(1) client~$i$ does not lose much if they send truthful updates, i.e. whenever all clients $j \neq i$ are reporting truthfully, then truthful reporting is $\varepsilon$-close to the optimal utility for client~$i$ (see Definition~\ref{definition:FL-BIC}), and (2) client~$i$'s best strategy is not too different from their true update, i.e. the best response strategy of client~$i$ when everyone else is reporting truthfully is $\varepsilon$-close to truthful reporting (see Definition~\ref{definition:truthful-reporting}). Below we formally define these two desirable properties for a federated learning protocol.\footnote{Appendix~\ref{appendix:BIC} compares the usual definition of Bayesian Incentive Compatibility with the one in this section. Usually the definition is phrased in terms of player types, which determine the payoff matrix/valuation function. Here we consider an equivalent formulation in terms of step-wise strategies and the final utility they produce.}
\begin{definition}[$\varepsilon$-Bayesian Incentive Compatibility]\label{definition:FL-BIC}
    A federated learning protocol $M$ is $\varepsilon$-\emph{Bayesian Incentive Compatible} (BIC) if:
    \begin{align*}
        \EE{U_i^M\left(\{\mathbf{1}_j\}_{j = 1}^N\right)}
        \geq \EE{U_i^M\left(\mathbf{s}_{i}, \{\mathbf{1}_j\}_{j \not= i}\right)} - \varepsilon,
    \end{align*}
    where $\mathbf{1}_i = (1, \dots, 1) \in \RR^T$ denotes fully truthful participation by client $i$, $\mathbf{s}_i = \left\{s_t^i\right\}_{t = 1}^T$ is some arbitrary strategy of client $i$, where $s_t^i$ is the strategy used by client $i$ at step $t$, and $U_i^M(\mathbf{v})$ is the utility of client~$i$ from $M$ when clients are using the strategy profile $\mathbf{v}$. Note that expectation is taken over the randomness in the clients' distributions, and any randomness (possibly none) in the protocol.
\end{definition}

In other words, a protocol is $\varepsilon$-Bayesian Incentive Compatibility if there exists a $\varepsilon$-Bayesian Nash Equilibrium of the game, where all participants are reporting truthfully their gradient updates. Therefore, for each client reporting truthfully is $\varepsilon$-close to best responding to all other clients reporting truthfully.\footnote{Moreover, the Relevation Principle from Mechanism Design states that any mechanism, where players are not necessarily being truthful, can be implemented by an incentive compatible mechanism with the same equilibrium payoffs~\citep{Nisan_Roughgarden_Tardos_Vazirani_2007}.}

\begin{definition}[$\varepsilon$-Approximately truthful reporting]\label{definition:truthful-reporting}
A strategy $(a^i_t, b^i_t)$ of client $i$ is \emph{approximately truthful} if it satisfies $\EE{\norm{a_t^i g_i(\theta_t) - g_i(\theta_t)}^2} \leq \varepsilon^2$ and $b^i_t \leq \varepsilon$.\footnote{Note that it is not truly necessary to have $b^i_t \leq \varepsilon$. We can assume that there is a global constant $b$, such that all clients can add a mean-zero variance-one noise term scaled by at most $b$, and the condition becomes $b^i_t \leq b + \varepsilon$} Moreover, in our analysis we require that the best response of client $i$ to truthful participation by clients $j \neq i$ is approximately truthful.
\end{definition}

\subsection{A note on reward functions}\label{subsection:reward-function}
Notice that our only constraint on the reward function $R_i(\theta)$ is for it to be $L$-Lipschitz. This allows our framework to cover a wide range of reward scenarios of practical interest. Below, we present several examples, starting with our primary motivation of modeling selfishness of heterogeneous clients.

\paragraph{Selfishness due to heterogeneity.}
Our primary motivation is to model clients who are selfish due to heterogeneity---they would like the final model to be as good as possible on their data distribution and the heterogeneity of client distributions creates an incentive to deviate from truthful participation. Mathematically, we can write this as $R_i(\theta) \sim -F_i(\theta)$, where we use $\sim$ to denote that $R_i(\theta)$ grows with $-F_i(\theta)$, i.e. as $F_i(\theta)$ decreases, $R_i(\theta)$ increases. For example, we could capture the phenomenon of \emph{diminishing returns} by modeling the reward for client $i$ with the logistic function $R_i(\theta) = \frac{1}{1 + e^{-1/F_i(\theta)}}$~\citep{Girdzijauskas_Štreimikienė_2007}.

\paragraph{Inter-client dynamics.}
We can model cooperative behavior if the reward $R_i(\theta) \sim \frac{-1}{\abs{S_i}} \sum_{j \in S_i} F_j(\theta)$ depends on the average loss of the clients in the group $S_i$ that client $i$ belongs to (akin to~\citet{donahue2020model, donahue2021optimality}). Moreover, we can also model antagonistic behavior similar to~\citet{dorner2024incentivizing} by setting the reward $R_i(\theta) \sim -\alpha F_i(\theta) + \frac{\beta}{N - 1} \sum_{j \not= i} F_j(\theta)$ to depend on some weighted sum of the model performance for client~$i$ and the negative of the model performance for the other clients (so client $i$ want to do well, while others shouldn't).

\paragraph{Additional fixed utilities.}
Because the Lipschitz property is oblivious to translations (in $\RR^d$), the reward function can accommodate situations where clients have fixed upfront utility that is independent of the final model $\theta$ and the history of the protocol: (1) upfront costs, (2) incentives to participate, (3) incentives to collect and transmit quality data, (4) subsidies. In particular, because of this our setup complements models that seek to handle free-riding~\citep{karimireddy2022mechanisms, han2023effect} and quality data collection~\citep{blum_one_2021}.

\section{Theoretical results}\label{section:theoretical-results}

In this section, we study a payment scheme which we prove is both $\varepsilon$-BIC and incentivizes approximately truthful reporting. We also provide explicit bounds on the penalties a client may pay and on the achieved rates of convergence of the global model, whenever the clients are approximately truthful. Finally, we discuss trade-offs between truthfulness, payment size and convergence.

At each step $t$ the server charges client $i$ the payment:
\begin{align}\label{eqn:payment_rule}
    p_t^i(\vec{m}_t) = C_t\left[\norm{m_t^i}^2 - \frac{1}{N - 1} \sum_{j \not= i} \norm{m_t^j}^2\right],
\end{align}
where $C_t$ is some client-independent constant (see the individual results below for definition). The total payment for each client for the protocol is then $p_i = \sum_{i=1}^T p_t^i(\vec{m}_t)$.
Note that this payment rule is \emph{budget-balanced}, that is at each step the server neither makes nor loses money because $ \sum_{i = 1}^N p_t^i(\vec{m}_t) = 0$, so it is possible that some clients get paid while others get charged according to what update they communicate. The budget-balanced property guarantees that there are no conflicting incentives with the server, because the server makes no profit from the payment scheme. Moreover, the payments can be computed solely based on the clients' messages, so no additional communication is required.

\subsection{Incentive compatibility and approximately truthful reporting}

First we show that the payment scheme is both $\varepsilon$-Bayesian Incentive Compatible and incentivizes approximately truthful reporting.

\begin{theorem}[Properties of the payment scheme]\label{theorem:eps-BIC}
    For every client $i$, suppose their objective $F_i$ is $H$-smooth and $m$-strongly-convex, and their reward $R_i$ is $L$-Lipschitz.
    Set $C_t = \frac{\sqrt{2 \mathcal{C}_t} \gamma_t L}{N \varepsilon}$, where $\mathcal{C}_t = \prod_{l = t + 1}^T c_{l}$ and $\ds c_l = 2 \left(1 - 2\gamma_l m + \gamma_l^2 H^2\right)$. Then
    \begin{enumerate}[nolistsep]
        \item the FedSGD protocol with Payment Rule (\ref{eqn:payment_rule}) is $\frac{\sqrt{2}LG \varepsilon}{N}$-BIC, where $G = \sum_{t = 1}^T \gamma_t \sqrt{\mathcal{C}_t}$,
        \item the best response strategy $(a_t^i, b_t^i)$ of client~$i$ to truthful participation from everyone else satisfies \\ $\ds \EE{\norm{a_t^i g_i(\theta_t) - g_i(\theta_t)}^2} \leq \varepsilon^2$ and $b_t^i \leq \varepsilon$ for all~$t$.
    \end{enumerate}
\end{theorem}

Intuitively, the two results demonstrate that if all other players $j\neq i$ are truthful, then it is in the interest of player $i$ to be (approximately) truthful as well, and this will yield close to optimal utility. Note that above the variables $C_t$ and $\mathcal{C}_t$ depend on the learning rate, so for different learning schedules we get different payment values. In our experiments in Section~\ref{section:experimental-results} we instead use constant $C_t$ and still obtain positive results.

\subsection{Payments and convergence}

Next, we provide explicit upper bounds on the total penalty paid by each player, as well as on the convergence speed for the loss function of each client, whenever the clients are $\varepsilon$-approximately truthful. These bounds allow us to discuss the interplay between learning quality and penalties, as well as their dependence on the parameters of the FL protocol.

\begin{theorem}[Bound on individual payments]\label{theorem:bound-on-payments}
    Suppose all participants are reporting approximately truthfully at each time step.
    Then the total payment paid by client $i$ is bounded by
    \begin{align*}
        \sum_{t = 1}^T p_t^i(\vec{m}_t)
        \leq   \frac{\sqrt{2}L G}{N} \left[2\varepsilon^2 + 2 \varepsilon \sigma + 2\zeta^2 + \rho^2\right] + \frac{\sqrt{8}L\varepsilon}{N} \sum_{t = 1}^T \gamma_t \sqrt{\mathcal{C}_t} \norm{\nabla F_i(\theta_t)},
    \end{align*}
    where $G = \sum_{t = 1}^T \gamma_t \sqrt{\mathcal{C}_t}$.
\end{theorem}

\begin{theorem}[Convergence rate]\label{theorem:convergence-of-sgd}
    Suppose all players are reporting $\varepsilon$-approximately truthfully at each time step.
    Assume that there exist scalars $M \geq 0$ and $M_V \geq 0$, such that for every $t$ the gradient error is bounded by \[\EE{\norm{e_i(\theta_t)}^2} \leq M + M_V \norm{\nabla F_i(\theta_t)}^2.\] Let $\eta = \frac{4 H (2 M_V + N)}{mN}$ and $\gamma_t = \frac{4}{m (\eta + t)}$.
    Then
    \begin{align*}
        \EE{F(\theta_T) - F(\theta^*)}
        \leq   \max\left\{\frac{16 H (2\varepsilon^2 + M + M_V \zeta^2)}{3 N m^2(\eta + T)}, \frac{(\eta + 1)(F(\theta_1) - F(\theta^*))}{\eta + T} \right\}.
    \end{align*}
\end{theorem}

\paragraph{Discussion.}
The following two scenarios explore the effect of higher levels of heterogeneity on the expected payment of each player and the convergence rate of the federated learning protocol. In particular, we seek to understand what is the trade-off between heterogeneity and payments/convergence rate. Both of these directly rely on Theorems~\ref{theorem:eps-BIC} and~\ref{theorem:convergence-of-sgd}.

\begin{example}[Constant heterogeneity bounds]
    Suppose $\zeta$ and $\rho$ are both constants fixed before the learning process is ever run. Then the total payment made by each player is at most $O\left(\frac{1}{N}\right)$ and the convergence rate becomes $O\left(\frac{1}{NT}\right)$. Hence, both the expected payments and the convergence rate decrease linearly in $N$, while preserving the $\varepsilon$-BIC property and the approximate truthful reporting property.
\end{example}

\begin{example}[Scaling the heterogeneity bounds for large $N$]
    Notice that even if the heterogeneity bounds increase with the number of clients, our results still give reasonable bounds on the individual payments and the convergence rate. Suppose $\zeta$ and $\rho$ are both of order $O\left(\sqrt[4]{N}\right)$. Then the total payment made by each player is at most $O\left(\frac{1}{\sqrt{N}}\right)$ and the convergence rate becomes $O\left(\frac{1}{\sqrt{N}T}\right)$. Hence, as we increase the number of participants $N$ we can simultaneously (1) increase the threshold for heterogeneity; and (2) reduce the maximal individual payment, while (3) preserving convergence.
\end{example}

\subsection{Extensions of the action space}\label{subsection:action-spaces-extension}
This subsection presents several generalization of the strategy space we defined in Subsection~\ref{section:game}. All three claims follow from our original theoretical analysis with slight modifications (see Appendix~\ref{appendix:extending-action-space} for details).

\begin{claim}[Mixed strategies]\label{claim:mixed-strategy}
    Suppose the message that client $i$ sends is $m^i_t = a^i_t g_i(\theta_t) + b^i_t \xi^i_t$, where $a^i_t$ and $b^i_t$ are random variables, such that $\abs{a^i_t} \geq 1$ almost surely and both variables are independent of the gradient updates and the protocol history. Then the results from Theorem~\ref{theorem:eps-BIC}, Proposition~\ref{theorem:bound-on-payments} and Theorem~\ref{theorem:convergence-of-sgd} still hold.
\end{claim}

\begin{claim}[History-dependent strategies]\label{claim:history-dependent-strategy}
    Suppose the message that client $i$ sends is $m^i_t = a^i_t g_i(\theta_t) + b^i_t \xi^i_t$, where $a^i_t$ and $b^i_t$ are random variables, such that $\abs{a^i_t} \geq 1$ almost surely and both variables are allowed to depend on the protocol history. Then the results from Theorem~\ref{theorem:eps-BIC}, Proposition~\ref{theorem:bound-on-payments} and Theorem~\ref{theorem:convergence-of-sgd} still hold.
\end{claim}

\begin{claim}[Directional strategies]\label{claim:allowing-angles}
    Suppose the message that client $i$ sends is $m^i_t = h^i_t + b^i_t \xi^i_t$, where
    \begin{itemize}[nolistsep]
        \item $h^i_t$ is a ``gradient'' vector, such that $\norm{h^i_t} = a^i_t g_i(\theta_t)$ and $\inner{h^i_t}{g_i(\theta_t)} \geq \norm{g_i(\theta_t)}^2$,\footnote{Cauchy-Schwarz gives $\abs{\inner{h^i_t}{g_i(\theta_t)}} \leq a^i_t \norm{g_i(\theta_t)}^2$.}
        \item $\xi^i_t$ is the additive noise term from before.
    \end{itemize}
    Then the results from Theorem~\ref{theorem:eps-BIC}, Proposition~\ref{theorem:bound-on-payments} and Theorem~\ref{theorem:convergence-of-sgd} still hold.
\end{claim}

\section{Experimental results}\label{section:experimental-results}
Finally, we present our experiments on three tasks from natural language processing and computer vision in an FL setting. Our goal is to evaluate the effectiveness of our payment scheme for FedSGD on non-convex, non-smooth objectives encountered in practice. Furthermore, we test the payment rule on a robust variant of FedSGD with coordinate-wise median aggregation and on the classic FedAvg protocol. Our results indicate that the incentives for gradient manipulation due to heterogeneity arise also in non-convex FL settings and that our payment mechanism is still successful in disincentivizing such manipulations from a client, provided that all others are truthful.

\begin{figure*}[htbp]
    \centering
    \subfloat[FeMNIST dataset]
    {
        \includegraphics[width=0.31\textwidth, trim=0mm 0mm 15mm 14mm, clip]{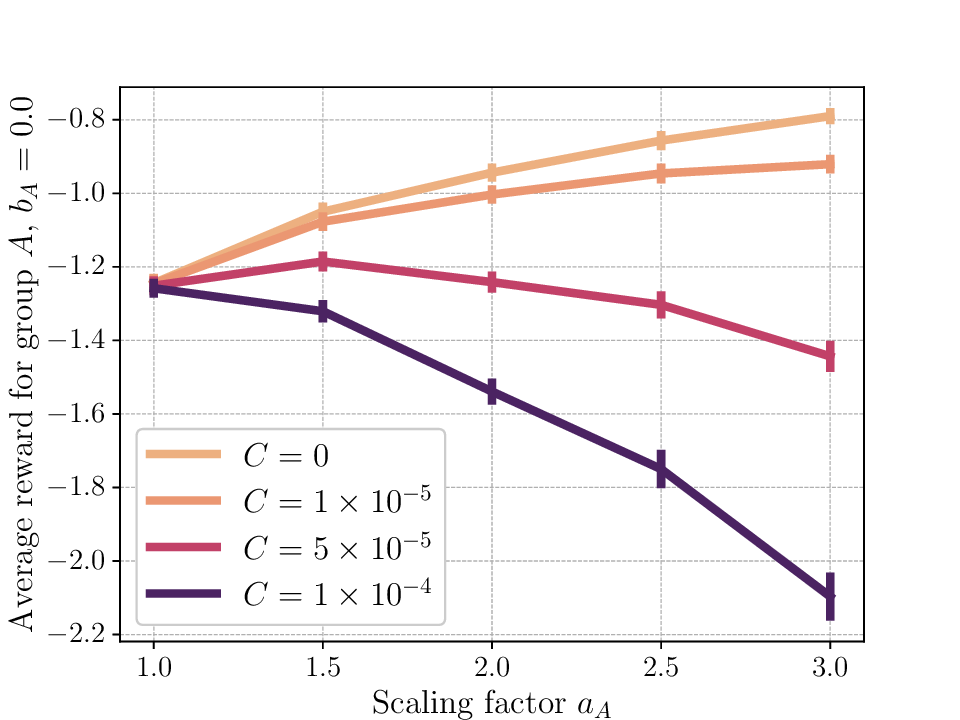}
        \label{figure:fedsgd-femnist}
    }
    \hfill
    \subfloat[Shakespeare dataset]
    {
        \includegraphics[width=0.31\textwidth, trim=0mm 0mm 15mm 14mm, clip]{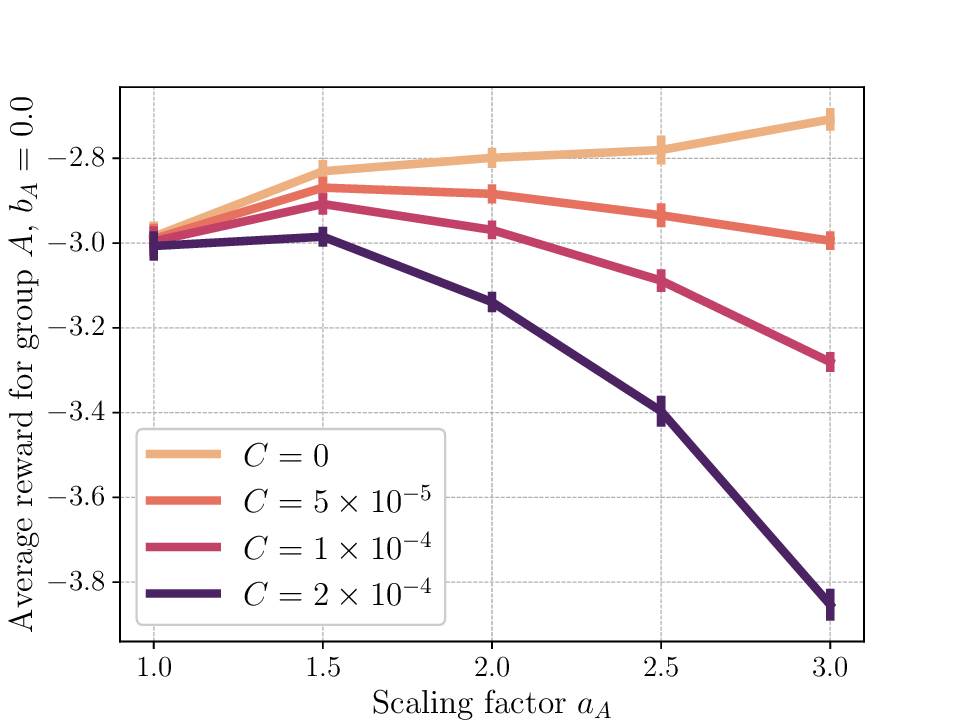}
        \label{figure:fedsgd-shakespear}
    }
    \hfill
    \subfloat[Twitter/Sent140 dataset]
    {
        \includegraphics[width=0.31\textwidth, trim=0mm 0mm 15mm 14mm, clip]{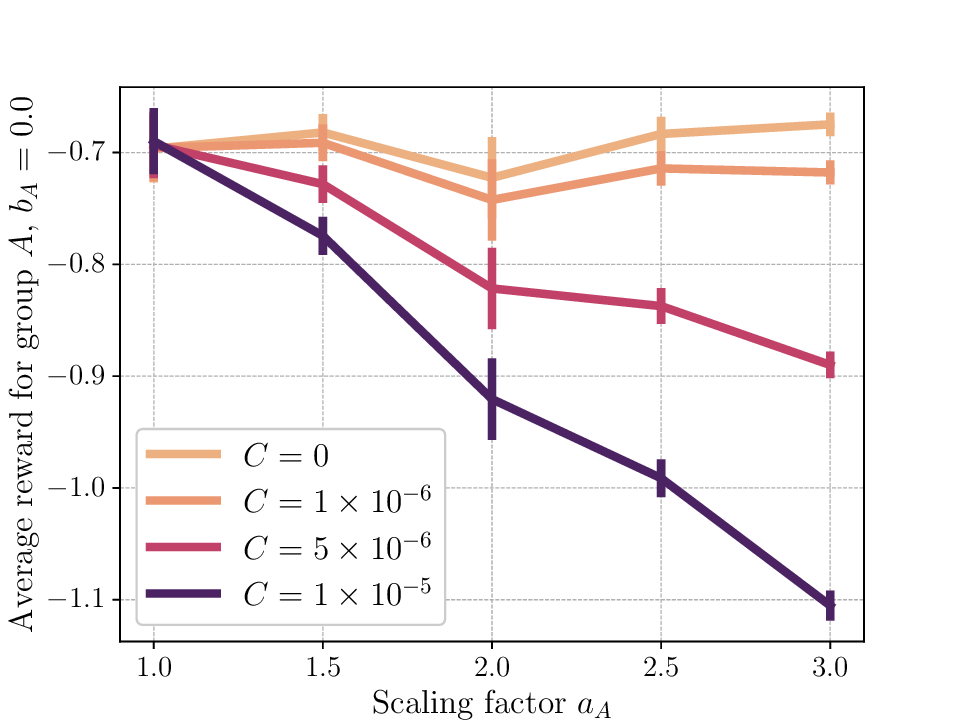}
        \label{figure:fedsgd-twitter}
    }
    \caption{The three plots illustrate the result of applying our payment scheme to the FedSGD protocol. Each line is the average of 10 runs of FedSGD, and the error bars show standard error. The constant $C$ controls the magnitude of the payment; smaller $C$ corresponds to smaller payment. The success of our mechanism is particularly prominent for the FeMNIST and Shakespeare datasets. For the Twitter dataset, the scaling factor is marginally beneficial to begin with. In all three experiments the misreporting client is only amplifying their gradient, without adding noise (so $b_A = 0.0$). For experiments with various levels of noise see Appendix~\ref{appendix:experiments}.}
    \label{figure:fedsgd}
\end{figure*}

\subsection{Tasks, data and training}

\paragraph{Image classification.\ } We use a two-layer Convolutional Neural Network. The first layer has 3 input channels and 32 output channels, while the second layer has 32 input channels and 64 output channels. Both layers apply a $5 \times 5$ kernel with stride $1$ and padding 2. After each layer we add ReLU activation and $2 \times 2$ max pooling with stride 2. We train on the FeMNIST dataset~\citep{caldas2018leaf} for $T = 10650$ steps with constant learning rate $\gamma = 0.06$.

\paragraph{Sentiment analysis.} We use a two-layer linear classifier with 384 hidden neurons on top frozen BERT embeddings. We train on the Twitter dataset~\citep{caldas2018leaf} for $T = 3550$ steps with constant learning rate $\gamma = 0.06$.

\paragraph{Next-character prediction.} We use a two-layer LSTM with embedding dimension of 8, 80 classes and 256 hidden units per layer. We train on the Shakespeare dataset~\citep{caldas2018leaf} for $T = 3550$ steps with learning rate $\gamma = 0.06$.

\paragraph{Data generation.} We use the federated learning datasets from ~\citet{caldas2018leaf}. We generate a collection of clients, each with their own training ($80\%$ of data) and test ($20\%$ of data) dataset.\footnote{We generate the maximal number of clients possible: $3597$ for FeMNIST, $2153$ for Twitter, and $800$ for Shakespeare.} Each client's dataset is generated non-iid from the complete data, ensuring data heterogeneity.

\paragraph{Client actions during training.} The following holds for all experimental tasks. In our evaluations, we work with $N = 3$ meta-clients obtained by grouping the datasets provided by LEAF~\citep{caldas2018leaf} into three equal parts, ensuring that each client gets diverse heterogeneous data. We choose a small constant number of clients because it allows for a simple visualization of the performance of our scheme. At each time step, one of these clients modifies their gradient according to a fixed strategy (determined by the individual experiment), while the other two truthfully send their gradients. Let $A$ be the set of misreporting clients, in this case $\abs{A} = 1$, and $B$ be the set of truthful clients, in this case $\abs{B} = 2$. Let $a_A$ and $b_A$ be the strategy parameters chosen by the client in $A$. For each client we calculate their utility as $U_i = - F_i(\theta_{\text{final}}) - C \cdot \sum_{t = 1}^T p^i_t$, with the test loss $F_i(\theta_{\text{final}})$ evaluated on each client's test data.

\begin{figure*}[htbp]
    \centering
    \subfloat[Performance of FedSGD with median-based aggregation on the FeMNIST dataset under our payment schemes. The magnitude of the payment is controlled by the constant $C$. Each line is the average of 10 runs, and the error bars show standard error.]
    {
        \includegraphics[width=0.45\linewidth, trim=0mm 0mm 15mm 14mm, clip]{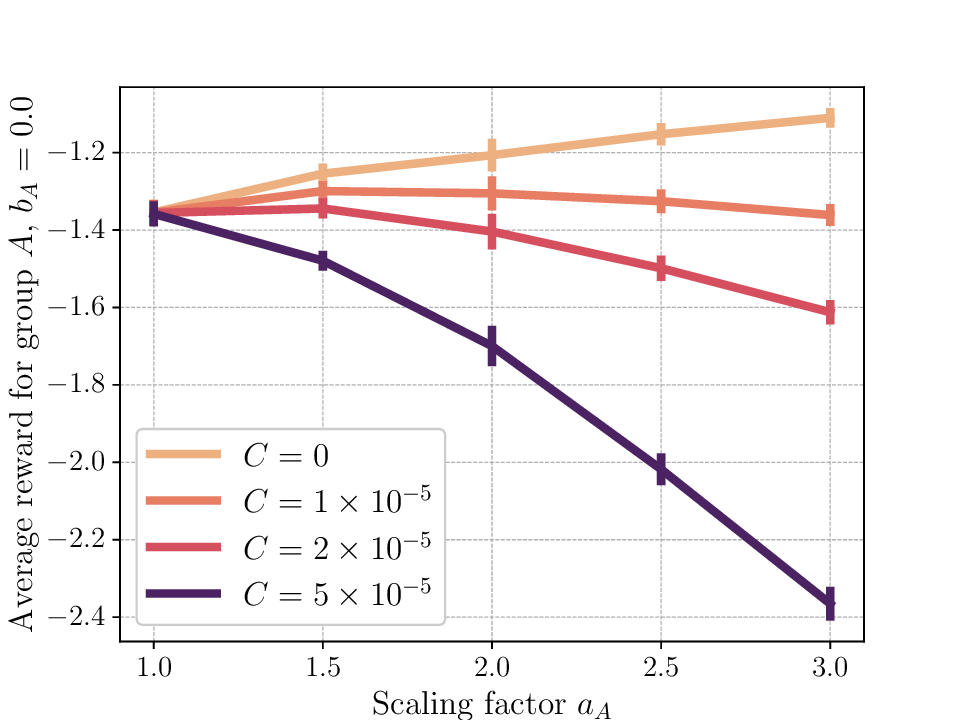}
        \label{figure:fedsgd-median}
    }
    \hfill
    \subfloat[Performance of FedAvg on the FeMNIST dataset under our payment schemes. The magnitude of the payment is controlled by the constant $C$. Each line is the average of 10 runs, and the error bars show standard error.]
    {
        \includegraphics[width=0.45\linewidth, trim=0mm 0mm 15mm 13mm, clip]{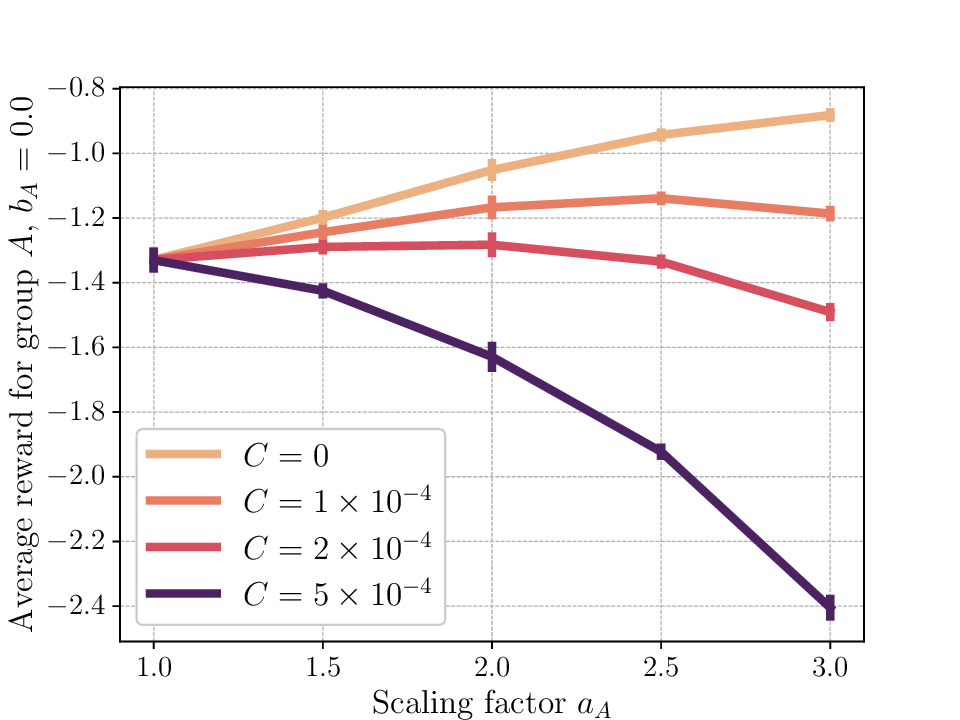}
        \label{figure:fedavg}
    }
    \caption{Experiments with other federated learning protocols}
    \label{figure:other-protocols}
\end{figure*}

\subsection{Discussion}
We visualize our results for FedSGD in Figure~\ref{figure:fedsgd}. We plot the utility of the clients in $A$ versus the scaling factor $a_A$ they use to upscale their updates. Here $b_A = 0$. The different curves correspond to the utility for different values of the penalty parameter $C$.

We see that our mechanism successfully disincentivizes gradient modification for the FeMNIST and Shakespeare datasets. This is because the utility of group $A$ decreases with $a_A$ for the curves with larger values of $C$, so the players in $A$ do not benefit from up-scaling their updates. For the Twitter dataset, the scaling factor is marginally beneficial to begin with, but the payment rule is still effective.

These results continue to hold for median-based FedSGD (see Figure~\ref{figure:fedsgd-median}) and FedAvg (see Figure~\ref{figure:fedavg}). We note that the effects of gradient modification are smaller for median-based FedSGD compared to standard FedSGD and FedAvg, which is expected because coordinate-wise median aggregation has been used as a benchmark for robust aggregators~\citep{pillutla2022robust}.

We refer to Appendix~\ref{appendix:experiments} for results from additional experiments with median-based FedSGD and FedAvg, as well as experiments with varying noise levels (instead of scaling factors).

\section{Conclusion}\label{section:conclusion}
In this work we showed that update manipulation incentives in federated learning can arise due to data heterogeneity, even without adversarial intent. To mitigate this, we introduced a game-theoretic framework and a budget-balanced payment mechanism that ensures approximate incentive compatibility and approximate truthful reporting. Empirical work confirmed the effectiveness of our method across multiple tasks and its adaptability to multiple FL protocols.

\section*{Impact Statement}
This paper presents work whose goal is to advance the field of Machine Learning. Our theoretical analysis contributes to better understanding of the impact of machine learning on society. Therefore, we expect that our results can serve a positive purpose in increasing the awareness about ML impact and encouraging further research on related topics. There are many potential societal consequences of our work, none which we feel must be specifically highlighted here.

\section*{Acknowledgements}
This research was partially funded by the Ministry of Education and Science of Bulgaria (support for INSAIT, part of the Bulgarian National Roadmap for Research Infrastructure). The authors would like to thank Florian Dorner and Kumar Kshitij Patel for helpful discussions at various stages of the project.

\bibliography{bibliography}
\bibliographystyle{plainnat}

\newpage
\clearpage
\onecolumn
\appendix

\section*{Summary of supplementary materials}
\begin{itemize}[noitemsep, topsep=0pt, parsep=0pt, partopsep=0pt]
    \item Appendix~\ref{appendix:definitions-refresher} is an overview of important definitions.
    \item Appendix~\ref{appendix:heterogeneity-assumptions} gives background on heterogeneity assumptions.
    \item Appendix~\ref{appendix:BIC} makes a brief remark on incentive compatibility.
    \item Appendix~\ref{appendix:mean-estimation} contains all our results on mean estimation from Section~\ref{section:motivating-examples}.
    \item Appendix~\ref{appendix:approximate-truthfulness} contains the proof of Theorem~\ref{theorem:eps-BIC}.
    \item Appendix~\ref{appendix:payment} contains the proof of Theorem~\ref{theorem:bound-on-payments}.
    \item Appendix~\ref{appendix:convergence} contains the proof of Theorem~\ref{theorem:convergence-of-sgd}.
    \item Appendix~\ref{appendix:extending-action-space} proves the action space extensions from Subsection~\ref{subsection:action-spaces-extension}.
    \item Appendix~\ref{appendix:gradient-norm} is a miscellaneous fact on gradient norm.
    \item Appendix~\ref{appendix:experiments} presents additional plots from our experiments.
\end{itemize}

\section{Theoretical refresher}\label{appendix:definitions-refresher}

\begin{definition}[Lipschitzness]
    A function $f : \calX \subseteq \RR^n \to \RR^m$ is \emph{$L$-Lipschitz} if
    \[
        \norm{f(x) - f(y)}_m \leq L \norm{x - y}_n
    \]
    for all $x, y \in \calX$.
\end{definition}

\begin{definition}[Smoothness]
    A differentiable function $f : \calX \subseteq \RR^d \to \RR$ is \emph{$H$-smooth} if
    \[
        \norm{\nabla f(x) - \nabla f(y)} \leq H \norm{x - y},
    \]
    for all $x, y \in \calX$. In other words, $f$ is $H$-smooth if its gradient $\nabla f$ is $H$-Lipschitz. Moreover, this condition is equivalent to
    \[
        \abs{f(x) - f(y) - \inner{\nabla f(y)}{x - y}} \leq \frac{H}{2} \norm{x - y}^2.
    \]
\end{definition}

\begin{definition}[Strong convexity]\label{definition:strong-convexity}
    Let $f : \calX \subseteq \RR^d \to \RR$ be a differentiable function. Then $f$ is \emph{$m$-strongly-convex} if
    \[
        f(x) \geq f(y) + \inner{\nabla f(y)}{x - y} + \frac{m}{2} \norm{x - y}^2.
    \]
\end{definition}

\begin{proposition}
    Let $f : \calX \subseteq \RR^d \to \RR$ be a differentiable function that is both $H$-smooth and $m$-strongly-convex, then
    \[
        \frac{m}{2} \norm{x - y}^2 \leq f(x) - f(y) - \inner{\nabla f(y)}{x - y} \leq \frac{H}{2} \norm{x - y}^2,
    \]
    In other words, the error between $f$ and its linear approximation is bounded by quadratics from both above and below.
\end{proposition}

\section{Heterogeneity assumptions from previous works}\label{appendix:heterogeneity-assumptions}

For this section assume we work on some space $\Theta \subseteq \RR^d$. In the literature, many heterogeneity assumptions have been formulated mainly for the study of convergence of different flavors of distributed SGD~\citep{koloskova_unified_2020, woodworth_minibatch_2020, patel2024limits}. Our focus here is to compare Assumption~\ref{assumption:gradient-difference} to the most relevant assumptions from the literature. To be the best our knowledge constraints similar to Assumption~\ref{assumption:variance-difference} that bound the heterogeneity of gradient variance have not been extensively studied.

\begin{assumption}[Bounded First-Order Heterogeneity,~\citep{koloskova_unified_2020, woodworth_minibatch_2020}]\label{assumption:bounded-first-order-heterogeneity}
    For any client $i$ and any $\theta \in \Theta$ we have:
    \[
        \sup_{\theta \in \Theta, i \in [N]}\norm{\nabla F_i(\theta) - \nabla F(\theta)}_2^2 \leq \zeta^2.
    \]
\end{assumption}

\begin{assumption}[Relaxed First-Order Heterogeneity,~\citep{karimireddy_scaffold_2020}]\label{assumption:relaxed-first-order-heterogeneity}
    There exist constants $G \geq 0$ and $D \geq 1$, such that for every $\theta \in \Theta$ the following holds:
    \[
        \frac{1}{N} \sum_{i = 1}^N \norm{\nabla F_i(\theta)}^2 \leq G^2 + D^2 \norm{\nabla F(\theta)}^2.
    \]
\end{assumption}

One can interpret Assumption~\ref{assumption:gradient-difference} as a relaxed version of Assumption~\ref{assumption:bounded-first-order-heterogeneity}, where we ignore the angular (or directional) heterogeneity between two gradients, and only focus on their difference in magnitude. On the other hand, Assumption~\ref{assumption:relaxed-first-order-heterogeneity} relaxes Assumption~\ref{assumption:gradient-difference} even further, by putting a constraint only on the average magnitude of the communicated gradients, as opposed to having a bound for each individual client. This positions our assumption in-between these two widely used heterogeneity constraints.

\begin{remark}
    If Assumption~\ref{assumption:relaxed-first-order-heterogeneity} holds with $D = 0$ and $\ds \theta^* \in \argmin_{\theta \in \Theta} F(\theta)$, then we recover another assumption used in the literature---\emph{Bounded First-Order Heterogeneity at Optima}~\citep{koloskova_unified_2020,woodworth_minibatch_2020,glasgow_sharp_2022}.
\end{remark}

\section{Note on Incentive Compatibility}\label{appendix:BIC}

The usual definition of Bayesian Incentive Compatibility from Mechanism Design is the following. Suppose that each player $i$ has a true type $t_i \in T_i$, and a utility function $u_i(t, o)$ that takes as input a type $t$ and an outcome $o$, and outputs a value. Let $o(t_i, t_{-i})$ be the outcome of the mechanism $M$ on input $(t_i, t_{-i})$. Let $p_i(t_i, t_{-i})$ be the payment of player $i$ according to $M$.

\begin{definition}[Bayesian Incentive Compatibility]\label{definition:BIC}
    A mechanism $M$ is \emph{Bayesian Incentive Compatible} (BIC) if:
    \[
        \EE{u_i(t_i, o(t_i, t_{-i})) - p_i(t_i, t_{-i})} \geq \EE{u_i(t_i, o(t'_i, t_{-i})) - p_i(t'_i, t_{-i})},
    \]
    where $t'_i \in T_i$ is any type for player $i$, and expectation is taken with respect to the randomness in the types of everyone but player $i$.

    Moreover, we can relax the condition to have an approximately BIC mechanism up to an additive constant, denoted $\varepsilon$-BIC:
    \[
        \EE{u_i(t_i, o(t_i, t_{-i})) - p_i(t_i, t_{-i})} \geq \EE{u_i(t_i, o(t_i, t_{-i})) - p_i(t'_i, t_{-i})} - \varepsilon.
    \]
\end{definition}
\section{Results on Mean Estimation}\label{appendix:mean-estimation}
\newcommand{\eqsign}{\operatorname*{eq}}

\subsection{Proof of Proposition~\ref{proposition:mean_example}}

One can easily check that if the first player is truthful, then the MSE is as stated.
Now, consider the case where the first client can lie by selecting a constant $c>1$.
In particular, if we select a constant $c$ such that
\[
1 < c < \frac{2\mu_1(\mu_1 - \mu)N + \mu_1^2 - \sigma^2}{\mu_1^2 + \sigma^2}
,\]
then the we have $\mathbb{E}\left[ \overline{\mu} \right] = \frac{c-1}{N}\mu_1 + \mu$ and $Var(\overline{\mu}) = \frac{\sigma^2}{N} + \frac{c^2-1}{N^2}\sigma^2$. Then the MSE for the first player is
\begin{align*}
\mathbb{E}\left[ \left( \overline{\mu} - \mu_1 \right) ^2 \right] &= \left( \mu + \frac{c-N-1}{N}\mu_1 \right) ^2 + \frac{\sigma^2}{N} + \frac{c^2-1}{N^2}\sigma^2 \\
    &= \left( \mu-\mu_1 \right) ^2 + 2 \frac{c-1}{N}\mu_1(\mu - \mu_1) + \frac{(c-1)^2}{N^2}\mu_1^2 + \frac{c^2-1}{N^2}\sigma^2 + \frac{\sigma^2}{N} \\
    &= \left( \mu-\mu_1 \right) ^2 + \frac{\sigma^2}{N} + \frac{c-1}{N^2} \left( 2\mu_1\left( \mu-\mu_1 \right)N + \left( c-1 \right)\mu_1^2 + \left( c+1 \right)\sigma^2 \right)  \\
    &=  \left( \mu-\mu_1 \right) ^2 + \frac{\sigma^2}{N} + \frac{c-1}{N^2} \left( c\left( \mu_1^2 + \sigma^2 \right)  -2\mu_1\left( \mu_1-\mu \right)N -\mu_1^2 + \sigma^2 \right)
.\end{align*}
Note that the last term is a quadratic in $c$ and hence is minimized at
\[
c = \frac{\mu_1^2 + \mu_1(\mu_1 - \mu)N}{\mu_1^2 + \sigma^2} > 1
.\]
Note that $c>1$ is guaranteed by the assumed inequality $\mu_1(\mu_1 - \mu)>\sigma^2 / N$. In this case, the MSE of the client is
\[
\mathbb{E}\left[ (\bar{\mu} - \mu_1)^2 \right] = (\mu-\mu_1)^2 + \frac{\sigma^2}{N} -  \frac{(\sigma^2/N - \mu_1(\mu_1 - \mu))^2}{\mu_1^2+\sigma^2}
,\]
which completes our proof.

\subsection{Individually optimal estimators without strategic actions}
\label{appendix:bayes-mean}

Similarly to Proposition~\ref{proposition:mean_example}, suppose that $N$ clients seek to estimate their respective means $\mu_1,\ldots,\mu_{N} \in \mathbb{R}$ sampled from a prior $\mu_i \sim \calN\left(\mu, 1 / \tau\right)$, where $\mu \sim \calN\left(0, 1 / \tau_0\right)$. Furthermore, assume that each client $i$ receives samples $\forall j \in [n] \: x^j_{i} \sim \mathcal{N}(\mu_i, \sigma_i^2)$. First, we want to understand how much clients could benefit from each other's samples in this Bayesian setting.

\begin{proposition}
    \label{proposition:bayes-mean}
    Consider $N$ clients who seek to find estimators $\{\hat{\mu}_i\}_{i=1}^N$ for their local parameters $\{\mu_i\}_{i=1}^N$ that are good in terms of mean squared error $(\hat{\mu}_i - \mu_i)^2$. Each client $i$ has $n$ independent noisy observations of $\mu_i$, such that $x_i^j \sim \calN\left(\mu_i, \sigma_i^2\right)$ for every $j \in [n]$ and $i \in [N]$. Moreover, we assume that $\{\mu_i\}_{i=1}^N$ are related to each other through a prior $\mu_i \sim \calN\left(\mu, 1 / \tau\right)$, and the prior mean is distributed as $\mu \sim \calN\left(0, 1 / \tau_0\right)$.
    Given all samples, the means follow
    \[
        \mu_i \sim \calN\left(\frac{\tau_i
        \bar{x}_i + \beta_i \sum_{j \neq i} \rho_j \bar{x}_j}{\tau_i + \beta_i
        (\sum_{j \neq i} \rho_j + \tau_0)}, \left(\tau_i +
        \beta_i \left(\sum_{j \neq i} \rho_j
        + \tau_0\right)\right)^{-1}\right),
    \]
    where $\ds \tau_j = \frac{n}{\sigma_j^2}$, $\ds \rho_j = \frac{\tau \tau_j}{\tau + \tau_j}$, and $\ds \beta_i = \frac{\tau}{\sum_{j \neq i} \rho_j + \tau + \tau_0}$.
\end{proposition}

\begin{proof}
To derive the distribution of $\mu_i$, we should integrate the probability
density function $f$ over $\mu, \{\mu_j\}_{j \neq i}$. This function is
proportional to
\[
    f \propto \exp\left(\sum_{j=1}^N -\frac{\tau_j (\bar{x}_j -
    \mu_j)^2}{2} + \sum_{j=1}^N -\frac{\tau (\mu_j - \mu)^2}{2} - \frac{\tau_0
    \mu^2}{2}\right).
\]

First, we want to integrate over $\{\mu_j\}_{j \neq i}$. By collecting the
terms relevant to $\mu_j$, we get
\begin{align*}
    -\frac{\tau_j (\bar{x}_j - \mu_j)^2}{2} - \frac{\tau (\mu_j -
    \mu)^2}{2}
    &= -\frac{\tau_j (\bar{x}_j^2 - 2 \bar{x}_j \mu_j + \mu_j^2)}{2} -
    \frac{\tau (\mu_j^2 - 2 \mu_j \mu + \mu^2)}{2} \\
    &= -\frac{(\tau_j + \tau) \mu_j^2}{2} + (\tau_j \bar{x}_j + \tau \mu) \mu_j - \frac{\tau_j \bar{x}_j^2}{2} - \frac{\tau \mu^2}{2} \\
    &= -\frac{(\tau_j + \tau) \left(\mu_j - \frac{\tau_j \bar{x}_j + \tau \mu}{\tau_j + \tau}\right)^2}{2} + \frac{(\tau_j \bar{x}_j + \tau \mu)^2}{2(\tau_j + \tau)} - \frac{\tau_j \bar{x}_j^2}{2} - \frac{\tau \mu^2}{2}
\end{align*}
and
\begin{align*}
    \frac{(\tau_j \bar{x}_j + \tau \mu)^2}{2 (\tau_j + \tau)} - \frac{\tau_j \bar{x}_j^2}{2} - \frac{\tau \mu^2}{2}
    &= \frac{(\tau_j \bar{x}_j)^2 + 2 \tau \tau_j \bar{x}_j \mu + \tau^2 \mu^2}{2 (\tau_j + \tau)} - \frac{\tau_j \bar{x}_j^2}{2} - \frac{\tau \mu^2}{2}\\
    &= -\frac{\rho_j \mu^2}{2} + \rho_j \bar{x}_j \mu - \frac{\rho_j \bar{x}_j^2}{2}
\end{align*}
Thus, after integrating over $\mu_j$, we get that the density function $f$ will
be proportional to
\[
    f \propto \exp\left(\sum_{j \neq i} \left(-\frac{\rho_j
    \mu^2}{2} + \rho_j \bar{x}_j \mu - \frac{\rho_j \bar{x}_j^2}{2}\right) -
    \frac{\tau_i (\bar{x}_i - \mu_i)^2}{2} -\frac{\tau (\mu_i - \mu)^2}{2}  -
    \frac{\tau_0 \mu^2}{2}\right).
\]

Now, we want to integrate over $\mu$. By collecting the terms relevant to
$\mu$, we get
\begin{align*}
    -\frac{\tau (\mu_i - \mu)^2}{2} - \frac{\tau_0 \mu^2}{2} + \sum_{j \neq
    i} -\frac{\rho_j \mu^2}{2} + \rho_j \bar{x}_j \mu
    &= -\frac{(\tau + \tau_0 + \sum_{j \neq i} \rho_j) \mu^2}{2} +
    \left(\tau \mu_i + \sum_{j \neq i} \rho_j \bar{x}_j\right) \mu - \frac{\tau
    \mu_i^2}{2} \\
    &= -\frac{(\tau + \tau_0 + \sum_{j \neq i} \rho_j) \left(\mu -
    \frac{\tau \mu_i + \sum_{j \neq i} \rho_j \bar{x}_j}{\tau + \tau_0 +
    \sum_{j \neq i} \rho_j}\right)^2}{2} + \frac{(\tau \mu_i + \sum_{j \neq i} \rho_j
    \bar{x}_j)^2}{2 (\tau + \tau_0 + \sum_{j \neq i} \rho_j)} - \frac{\tau
    \mu_i^2}{2}
\end{align*}
Thus, the final density will be proportional to
\begin{align*}
    f &\propto \exp\left(\frac{\left(\tau \mu_i + \sum_{j \neq i} \rho_j
    \bar{x}_j\right)^2}{2 (\tau + \tau_0 + \sum_{j \neq i} \rho_j)} - \frac{\tau
    \mu_i^2}{2} - \frac{\tau_i (\bar{x}_i - \mu_i)^2}{2}\right)\\
    &= \exp\left(-\frac{\mu_i^2 \cdot \left(\frac{\tau (\tau_0 + \sum_{j \neq i}
    \rho_j)}{\tau + \tau_0 + \sum_{j \neq i} \rho_j} + \tau_i\right) }{2} +
    \left(\tau_i \bar{x}_i + \frac{\tau \sum_{j \neq i} \rho_j
    \bar{x}_j}{\tau + \tau_0 + \sum_{j \neq i} \rho_j}\right) \mu_i + \dots\right)\\
    &= \exp\left(-\frac{(\tau_i + \beta_i (\tau_0 + \sum_{j \neq i}
    \rho_j)) \left(\mu_i - \frac{\tau_i \bar{x}_i + \beta_i \sum_{j \neq
    i} \rho_j \bar{x}_j}{\tau_i + \beta_i (\tau_0 + \sum_{j \neq i}
    \rho_j)}\right)^2}{2} + \dots\right),
\end{align*}
which gives us the desired conditional distribution of $\mu_i$.

\end{proof}

\paragraph{Variance reduction incentives}

First, notice in the absence of heterogeneity, clients are not adverse to each
other. To see it, consider the situation when the means of all clients are the
same, $\tau \to \infty$. In this situation, individually optimal estimators
converge to the same estimator
\[
    \hat{\mu}^* = \frac{\sum_{j=1}^N \tau_j \bar{x}_j}{\tau_0 + \sum_{j=1}^N \tau_j}.
\]
This estimator is the best estimator for $\mu$ in the Bayesian setting. Thus,
in the absence of heterogeneity, the clients want cooperate with each other to
reduce the variance of individual estimators.

\paragraph{Bias personalization incentives}

However, if heterogeneity is present, the clients will also want to skew bias
of the estimator closer to their mean. To see this effect, consider the
situation when variance reduction incentives are symmetric, $\tau_i = \tau_j$ for all $i, j$. In this situation, we might notice that the weight of a
client's own samples becomes bigger than the weight the samples of any other
client in the individually optimal estimator,
\[
    \tau_i = \tau_j > \beta_i \rho_j.
\]
Thus, in the presence of heterogeneity, the clients want to increase the weight
of the own samples in the final model, which will bias the final model closer
to their own mean.

\subsection{Nash equilibrium in the presence of scaling attacks}
\label{appendix:game-eqilibrium}

Now, we consider a game-theoretic setting, where each client $i$ sends a message $m_{i} = \frac{c_i}{n} \sum_{j=1}^n x_i^j$ to the server. Then, the server computes an aggregate $\overline{\mu} = \frac{1}{N} \sum_{i=1}^{N} m_{i}$ and broadcasts it to all clients. Each client $i$ would like to receive an estimate of their local mean with minimal mean squared error $\mathbb{E}[ \left( \overline{\mu} - \mu_i \right)^2 ]$.

\begin{proposition}
    \label{proposition:game-eqilibrium}
    The Nash equilibrium of the game above is
    \[
        c_i^{\eqsign} = \frac{N \rho_i (1/\tau + 1/\tau_0)}{1 + (\sum_{j=1}^N
        \rho_j) / \tau_0},
    \]
    which results in the following final error:
    \[
        -R_i^{\eqsign} = \EE{(\overline{\mu} - \mu^i)^2} = \frac{(1/\tau + 1/\tau_0)
        (\sum_{j=1}^N \rho_j - 2 \rho_i + \tau)}{\tau (1 + \sum_{j=1}^N \rho_j
        / \tau_0)}.
    \]
\end{proposition}
\begin{proof}

Given actions profile, the error of client $i$ is
\begin{align*}
    \EE{(\overline{\mu} - \mu_i)^2}
    &=  \EE{\left(\sum_{j=1}^N \frac{c_j}{N} \bar{x}_j -
    \mu_i\right)^2} \\
    &=  \EE{\left(\sum_{j=1}^N \frac{c_j}{N} (\bar{x}_j -
    \mu_j) + \sum_{j=1}^N \frac{c_j}{N} \mu_j - \mu_i\right)^2} \\
    &=  \sum_{j=1}^N \frac{(c_j)^2}{N^2 \tau_j} +
    \EE{\left(\sum_{j=1}^N \frac{c_j}{N} \mu_j - \mu_i\right)^2} \\
    &=  \sum_{j=1}^N \frac{(c_j)^2}{N^2 \tau_j} + \EE{\left(\sum_{j \neq i} \frac{c_j}{N} (\mu_j - \mu) + (\mu_i - \mu) \cdot \left(\frac{c_i}{N} - 1\right) + \mu \cdot \left(\frac{\sum_{j=1}^N c_j}{N} - 1\right) \right)^2} \\
    &= \sum_{j=1}^N \frac{(c_j)^2}{N^2 \tau_j} + \frac{\sum_{j \neq i} (c_j)^2}{N^2 \tau} + \frac{1}{\tau} \cdot \left(\frac{c_i}{N} - 1\right)^2  + \frac{1}{\tau_0} \cdot \left(\frac{\sum_{j=1}^N c_j}{N} - 1\right)^2  \\
    &= \sum_{j=1}^N \frac{(c_j)^2}{N^2 \rho_j} + \frac{\left(\sum_{j=1}^N
    c_j\right)^2}{N^2 \tau_0} - \frac{2 c_i}{N \tau} - \frac{2 \sum_{j=1}^N
    c_j}{N \tau_0} + \frac{1}{\tau} + \frac{1}{\tau_0}.
\end{align*}
Thus, the best response of client $i$ satisfies
\[
    \frac{c_i^*}{\rho_i} + \frac{c_i^*}{\tau_0} + \frac{\sum_{j \neq i}
    c_j}{\tau_0} = \frac{N}{\tau} + \frac{N}{\tau_0}.
\]
So, in equilibrium, for every client $i$ we get
\[
    \frac{c_i^{\eqsign}}{\rho_i} + \frac{\sum_{j=1}^N c_j^{\eqsign}}{\tau_0}
    = \frac{N}{\tau} + \frac{N}{\tau_0}.
\]
Thus, $c_i^{\eqsign} = \rho_i \lambda$ and
\[
    \left(1 + \frac{\sum_{j=1}^N \rho_j}{\tau_0}\right) \cdot \lambda =
    \frac{N}{\tau} + \frac{N}{\tau_0},
\]
which gives us the desired formula.

\end{proof}

\paragraph{Penalty from personalization incentives}
Notice that, in the case of truthful reporting, we get
\[
    -R_i^0 = \EE{(\overline{\mu} - \mu_i)^2} = \frac{1}{N} \left(\sum_{j=1}^N
    \frac{1}{\rho_j} - \frac{2}{\tau}\right) + \frac{1}{\tau}.
\]
In the limit $N \to \infty$, we get
\[
    -R_i^{\eqsign} \to \frac{1}{\tau} + \frac{\tau_0}{\tau^2}, \: -R_i^0 \to
    \frac{1}{\tau} + \frac{1}{N} \sum_{j=1}^N \frac{1}{\rho_j}.
\]

Thus, in the case $\forall j \: \tau^2 + \tau_j \tau < \tau_j \tau_0$ (i.e., when the heterogeneity between the means of the clients is high), the equilibrium actions of the clients result in a larger error compared to the truthful reporting. Notice that the penalty to the reward increases
quadratically with the variance of heterogeneous means $\tau$. The clients suffer a big
utility loss when the heterogeneity between their distributions is big.

\paragraph{Equilibrium in the absence of heterogeneity}
In the absence of heterogeneity (i.e., $\tau \to \infty$), we get
\[
    c_i^{\eqsign} \to \frac{N \tau_i}{\tau_0 + \sum_{j=1}^N \tau_j},
\]
which are exactly the weights for the best estimator in Bayesian setting. Thus,
as we predicted in the previous subsection, in the absence of heterogeneity,
the clients do not have adverse incentives and collaborate with each other to
reduce the variance of their estimators.

\paragraph{Equilibrium in the case of homogeneous data quality}
When the quality of local data is the same for each client ($\tau_i = \tau_j$ for all $i, j$), the formula for the equilibrium weight simplifies
\[
    c_i^{\eqsign} = \frac{1 + \tau_0/\tau}{1 + \tau_0/(N\rho)}
\]
Here, we can see how bias personalization and variance reduction incentives
interact. When the quality of data is poor ($\rho$ is small), $c_i^{\eqsign}$ also
becomes small. The clients are not incentivized to skew model closer to
their local mean because they can not estimate the personal bias well. Instead,
they want to rely more on the data of others and global prior to reduce the
variance of their estimator. On the other hand, when the quality of data is
high ($\rho$ is big), $c_i^{\eqsign}$ becomes bigger than $1$. The clients want to
skew the global model in their favor. However, the size of this skew is limited
because the clients do not want the global model to diverge.
\section{Proof of Theorem~\ref{theorem:eps-BIC}}\label{appendix:approximate-truthfulness}
We start with a summary of the results in this section. A trajectory is the history of models produced by the FL protocol, i.e. the sequence $\theta_1, \theta_2, \dots, \theta_{T + 1}$. Combining all claims here establishes Theorem~\ref{theorem:eps-BIC}.
\begin{itemize}[noitemsep, topsep=0pt, parsep=0pt, partopsep=0pt]
    \item Claim~\ref{claim:bound-on-trajectories-at-time-t-plus-1} bounds the step-wise difference between two model trajectories, produced by different strategies.
    \item Claim~\ref{claim:bound-reward-total-at-time-T-plus-1} bounds the difference in total reward achieved at time $T + 1$, i.e. the reward from the final model, between two trajectories.
    \item Claim~\ref{claim:bound-on-payment-per-turn-hold-prev-fixed} bounds the difference
    \item Proposition~\ref{proposition:bound-utility-total-at-time-T-plus-1} combines the previous to show that under our payment scheme (with a good choice of constant $C_t$) the optimal per-step strategy for client $i$, when everyone else is truthful, is $\varepsilon$-close to the true one.
    \item Claim~\ref{claim:bound-on-reward-loss-per-turn} extends the previous result to the full run of the protocol.
    \item Combining the previous yields Theorem~\ref{theorem:eps-BIC}
\end{itemize}
\vspace{1em}
\begin{claim}[Per-turn bound on trajectory difference]\label{claim:bound-on-trajectories-at-time-t-plus-1}
    Fix a client $i$. Let $\theta = \{\theta_t\}_{t = 1}^{T + 1}$ and $\theta' = \{\theta'_t\}_{t = 1}^{T + 1}$, be two trajectories obtained from two distinct strategy profiles $\{a^i_t\}_{t = 1}^{T}$ and $\{\bar{a}^i_t\}_{t = 1}^{T}$ of client $i$, while everyone else is implementing the same strategy in both scenarios. Then at time $t + 1$ conditioned on the event that $\norm{\theta_{t + 1} - \theta'_{t + 1}}^2 \geq \varepsilon$ we have:
    \[
        \EE{\norm{\theta_{t + 1} - \theta'_{t + 1}}^2} \leq
        c_t \EE{\norm{\theta_t - \theta'_t}^2}
        + \frac{2\gamma_t^2}{N^2} \EE{(a^i_t - \bar{a}^i_t)^2 \norm{g_i(\theta'_t)}^2} + \frac{\gamma_t^2(b^i_t - \bar{b}^i_t)^2}{N^2},
    \]
    where $\ds c_t = 2 \left(1 - \frac{2 \gamma_t m A_t}{N} + \frac{\gamma_t^2 A_t^2 H^2}{N^2}\right)$ and $\ds A_t = \sum_{j = 1}^N a^j_t$.
\end{claim}
\begin{proof}
    Observe the following sequence
    \begin{align*}
        \EE{\norm{\theta_{t + 1} - \theta'_{t + 1}}^2}
        &=  \EE{\norm{(\theta_t - \theta'_t) - \gamma_t(\bar{m}_t - \bar{m}'_t)}^2} \\
        &=  \EE{\norm{(\theta_t - \theta'_t) - \frac{\gamma_t}{N} \sum_{j = 1}^{N} \left(a^j_t g_j(\theta_t) - \bar{a}^j_t g_j(\theta'_t)+ b^i_t \xi^i_t - \bar{b}^i_t \bar{\xi}^i_t \right)}^2} \\
        &=  \EE{\norm{\frac{(\bar{a}^i_t - a^i_t)\gamma_t}{N} g_i(\theta'_t) + (\theta_t - \theta'_t) - \frac{\gamma_t}{N} \sum_{j = 1}^{N} a^j_t  \left(g_j(\theta_t) - g_j(\theta'_t)\right)}^2} + \frac{\gamma_t^2(b^i_t - \bar{b}^i_t)^2}{N^2} \\
        &\leq   2\EE{\norm{\frac{(\bar{a}^i_t - a^i_t)\gamma_t}{N}g_i(\theta'_t)}^2} + 2\EE{\norm{(\theta_t - \theta'_t) - \frac{\gamma_t}{N} \sum_{j = 1}^{N} a^j_t  \left(g_j(\theta_t) - g_j(\theta'_t)\right)}^2} + \frac{\gamma_t^2(b^i_t - \bar{b}^i_t)^2}{N^2} \\
        &\leq  2\EE{\norm{(\theta_t - \theta'_t) - \frac{\gamma_t}{N} \sum_{j = 1}^{N} a^j_t  \left(g_j(\theta_t) - g_j(\theta'_t)\right)}^2} + \frac{2\gamma_t^2 (a^i_t - \bar{a}^i_t)^2 \EE{\norm{g_i(\theta'_t)}^2}}{N^2} + \frac{\gamma_t^2(b^i_t - \bar{b}^i_t)^2}{N^2} \\
        &=  2 \EE{\norm{\theta_t - \theta'_t}^2}
            + 2\EE{\norm{\frac{\gamma_t}{N} \sum_{j = 1}^{N} a^j_t  \left(g_j(\theta_t) - g_j(\theta'_t)\right)}^2} \\
            &\quad - 4\EE{\frac{\gamma_t}{N} \sum_{j = 1}^{N} a^j_t \inner{\theta_t - \theta'_t}{g_j(\theta_t) - g_j(\theta'_t)}} + \frac{2\gamma_t^2 (a^i_t - \bar{a}^i_t)^2 \EE{\norm{g_i(\theta'_t)}^2}}{N^2} + \frac{\gamma_t^2(b^i_t - \bar{b}^i_t)^2}{N^2} \\
        &\leq 2 \EE{\norm{\theta_t - \theta'_t}^2}
            + \frac{2\gamma_t^2 A_t^2 H^2}{N^2} \EE{\norm{\theta_t - \theta'_t}^2} - \frac{4 \gamma_t m A_t}{N} \EE{\norm{\theta_t - \theta'_t}^2}
            \\
            &\quad  + \frac{2\gamma_t^2 \EE{(a^i_t - \bar{a}^i_t)^2 \norm{g_i(\theta'_t)}^2}}{N^2} + \frac{\gamma_t^2(b^i_t - \bar{b}^i_t)^2}{N^2} \\
        &= 2\left(1 - \frac{2 \gamma_t m A_t}{N} + \frac{\gamma_t^2 A_t^2 H^2}{N^2}\right) \EE{\norm{\theta_t - \theta'_t}^2} + \frac{2\gamma_t^2}{N^2}  \EE{(a^i_t - \bar{a}^i_t)^2 \norm{g_i(\theta'_t)}^2} + \frac{\gamma_t^2(b^i_t - \bar{b}^i_t)^2}{N^2}
    \end{align*}
    The first three lines are rearrangement. The third line uses the fact that both $\xi^i_t$ and $\bar{\xi}^i_t$ are have mean zero and variance one. The forth line uses the inequality $(x + y)^2 \leq 2x^2 + 2y^2$.\footnote{We do this roundabout bounding, so that all terms are squared norms. Also notice that $(x + y)^2 \leq 2x^2 + 2y^2$ is equivalent to $0 \leq (x - y)^2$, which is trivially true.} On line seven (second to last inequality), the second term follows from $H$-smoothness of the loss function for the sample $z^j_t$ of client $j$ at time $t$,
    while the third term follows from the $m$-strong-convexity of $F_i$.
\end{proof}

\begin{corollary}[Bound on trajectory difference]\label{corollary:total-theta-bound-at-time-T-plus-1}
    Suppose that $\theta = \{\theta_t\}_{t = 1}^{T + 1}$ is obtained from fully truthful participation from all clients, i.e. $(a^i_t, b^i_t) = (1, 0)$. Then at time $T + 1$:
    \[
        \EE{\norm{\theta_{T + 1} - \theta'_{T + 1}}^2} \leq
        \frac{2}{N^2} \sum_{t = 1}^{T} \gamma_t^2 \mathcal{C}_t \left(\EE{(\bar{a}_t^i - 1)^2 \norm{g_i(\theta'_t)}^2} + \left(\bar{b}^i_t\right)^2\right),
    \]
    where $\ds \mathcal{C}_t = \prod_{l = t + 1}^T c_{l}$ and $\ds c_l = 2 \left(1 - 2\gamma_l m + \gamma_l^2 H^2\right)$.
\end{corollary}
\begin{proof}
    Apply Claim~\ref{claim:bound-on-trajectories-at-time-t-plus-1} and telescope. Because the reference strategy is all truthful reporting, then we have $A = N$ in Claim~\ref{claim:bound-on-trajectories-at-time-t-plus-1}.
\end{proof}

\begin{claim}[Bound on reward]\label{claim:bound-reward-total-at-time-T-plus-1}
    Suppose that $\theta = \{\theta_t\}_{t = 1}^{T + 1}$ is obtained from fully truthful participation from all clients. Then at time $T + 1$:
    \begin{align*}
        \EE{\abs{R_i(\theta_{T + 1}) - R_i(\theta'_{T + 1})}}
        &\leq   \frac{\sqrt{2} L}{N} \sum_{t = 1}^{T} \gamma_t \sqrt{\mathcal{C}_t}  \left(\sqrt{\EE{\left(\bar{a}_t^i - 1\right)^2 \norm{g_i(\theta'_t)}^2}} + \abs{\bar{b}^i_t}\right),
    \end{align*}
    where $\ds \mathcal{C}_t = \prod_{t' = 1}^t c_{t'}$ and $\ds c_l = 2 \left(1 - 2\gamma_l m + \gamma_l^2 H^2\right)$.
\end{claim}
\begin{proof}
    Observe:
    \begin{align*}
        \left(\EE{\abs{R_i(\theta_{T + 1}) - R_i(\theta'_{T + 1})}}\right)^2
        &\leq   L^2 \left(\EE{\norm{\theta_{T + 1} - \theta'_{T + 1}}}\right)^2 \\
        &\leq   L^2 \EE{\norm{\theta_{T + 1} - \theta'_{T + 1}}^2} \\
        &\leq   L^2 \mathcal{C}_0 \EE{\norm{\theta_1 - \theta'_1}^2}
            + \frac{2 L^2}{N^2} \sum_{t = 1}^{T} \gamma_t^2 \mathcal{C}_t \left(\EE{(\bar{a}_t^i - 1)^2 \norm{g_i(\theta'_t)}^2} + \left(\bar{b}^i_t\right)^2\right) \\
        &=      \frac{2 L^2}{N^2} \sum_{t = 1}^{T} \gamma_t^2 \mathcal{C}_t  \left(\EE{(\bar{a}_t^i - 1)^2 \norm{g_i(\theta'_t)}^2} + \left(\bar{b}^i_t\right)^2\right)
    \end{align*}
    The first line follows from $F_i(\cdot)$ being $L$-Lipschitz. The second line follows from Cauchy-Schwartz. The third line applies Corollary~\ref{corollary:total-theta-bound-at-time-T-plus-1}, and uses the assumption $\theta_1 = \theta'_1$. Finally, to get the desired inequality we observe that $\sqrt{x + y} \leq \sqrt{x} + \sqrt{y}$ for non-negative $x, y$.
\end{proof}

\begin{claim}[Bound on payment]\label{claim:bound-on-payment-per-turn-hold-prev-fixed}
    Suppose $\theta$ and $\theta'$ are such that (1) they coincide exactly for the first $t - 1$ steps, (2) they differ in the strategy used by client~$i$ at time $t$, so $(a_t^i, b_t^i) = (1, 0)$ and $\bar{a}_t^i \geq 1$, and (3) all clients are truthful for all step $> t$. Then
    \[
        \EE{p_t^i(\vec{m}_t) - p_t^i(\vec{m}'_t)} \leq - C_t \EE{(\bar{a}_t^i - 1)^2 \norm{g_i(\theta_t)}^2} - C_t (\bar{b}^i_t)^2.
    \]
\end{claim}
\begin{proof}
    For what's to follow keep in mind that $\bar{a}_t^i \geq 1$ by assumption, and the additive noise is mean zero and variance one. Observe:
    \begin{align*}
        \EE{p_t^i(\vec{m}_t) - p_t^i(\vec{m}'_t)}
        &=      C_t \EE{\norm{g_i(\theta_t)}^2} - C_t \EE{\frac{1}{N - 1} \sum_{j \not= i} \norm{g_j(\theta_t)}^2} \\
        & \quad - C_t \EE{\norm{\bar{a}_t^i g_i(\theta_t) + \bar{b}^i_t \xi_i^t}^2} + C_t \EE{\frac{1}{N - 1} \sum_{j \not= i} \norm{g_j(\theta_t)}^2} \\
        &=      C_t \EE{\norm{g_i(\theta_t)}^2} - C_t \EE{\frac{1}{N - 1} \sum_{j \not= i} \norm{g_j(\theta_t)}^2} \\
        & \quad - C_t \EE{\norm{\bar{a}_t^i g_i(\theta_t)}^2} - C_t (\bar{b}^i_t)^2 + C_t \EE{\frac{1}{N - 1} \sum_{j \not= i} \norm{g_j(\theta_t)}^2} \\
        &=      C_t \EE{\norm{g_i(\theta_t)}^2} - C_t \EE{\norm{\bar{a}_t^i g_i(\theta_t)}^2} - C_t (\bar{b}^i_t)^2\\
        &=      C_t \EE{(1 - (\bar{a}_t^i)^2) \norm{g_i(\theta_t)}^2} - C_t (\bar{b}^i_t)^2\\
        &\leq   - C_t \EE{(\bar{a}_t^i - 1)^2 \norm{g_i(\theta_t)}^2} - C_t (\bar{b}^i_t)^2
    \end{align*}
    The last line follows from the assumption $\bar{a}_t^i \geq 1$.
\end{proof}

\begin{proposition}[Bound on utility difference between trajectories]\label{proposition:bound-utility-total-at-time-T-plus-1}
    Suppose $\theta$ and $\theta'$ are such that (1) they coincide exactly for the first $t - 1$ steps, (2) they differ in the strategy used by client~$i$ at time $t$, so $(a_t^i, b_t^i) = (1, 0)$ and $\bar{a}_t^i > 1$, and (3) all clients are truthful for all step $> t$. Then
    \begin{align*}
        \EE{R_i(\theta'_{t + 1}) - p_t^i(\vec{m}'_t) - \left(R_i(\theta_{t + 1}) - p_t^i(\vec{m}_t)\right)}
        &\leq   \frac{\sqrt{2 \mathcal{C}_t}\gamma_t L}{N} \left(\sqrt{\EE{\left(\bar{a}_t^i - 1\right)^2 \norm{g_i(\theta_t)}^2}} + \abs{\bar{b}^i_t} \right) \\
        &\quad - C_t \left(\EE{(\bar{a}_t^i - 1)^2 \norm{g_i(\theta_t)}^2} + \left(\bar{b}^i_t\right)^2 \right).
    \end{align*}
    Moreover, for $\ds C_t = \frac{\sqrt{2 \mathcal{C}_t} \gamma_t L}{N \varepsilon}$ we have:
    \[
        \EE{R_i(\theta'_{T}) - p_t^i(\vec{m}'_t) - \left(R_i(\theta_{T}) - p_t^i(\vec{m}_t)\right)} \leq \frac{\sqrt{2 \mathcal{C}_t} \gamma_t L \varepsilon}{N},
    \]
    and the best strategy of client~$i$ is such that $\EE{\norm{\bar{a}_t^i g_i(\theta_t) - g_i(\theta_t)}^2} \leq \varepsilon^2$ and $\bar{b}_t^i \leq \varepsilon$.
\end{proposition}
\begin{proof}
    We choose to present the proof for $\bar{b}_t^i = 0$ to show the proof methodology and avoid making the notation burdensome; similar argument holds for positive $\abs{\bar{b}^i_t}$ and $\left(\bar{b}^i_t\right)^2$. The first inequality is a combination of Claim~\ref{claim:bound-reward-total-at-time-T-plus-1} and Claim~\ref{claim:bound-on-payment-per-turn-hold-prev-fixed}.\footnote{In particular, we can ignore the first $t - 1$ steps because they are identical for both trajectories, so the bound in Claim~\ref{claim:bound-reward-total-at-time-T-plus-1} is relevant only for steps $\geq t$.} Now we tackle the second portion.

    First, we write:
    \begin{align*}
        \EE{R_i(\theta'_{T}) - p_t^i(\vec{m}'_t) - \left(R_i(\theta_{T}) - p_t^i(\vec{m}_t)\right)}
        &\leq   \frac{\sqrt{2 \mathcal{C}_t} \gamma_t L}{N} \sqrt{\EE{\left(\bar{a}_t^i - 1\right)^2 \norm{g_i(\theta_t)}^2}} \\
        &\quad - C_t \EE{(\bar{a}_t^i - 1)^2 \norm{g_i(\theta_t)}^2}
    \end{align*}
    Next, the right-hand side expression is a downwards-curved quadratic, and has roots
    \[
        \sqrt{\EE{\left(\bar{a}_t^i - 1\right)^2 \norm{g_i(\theta'_t)}^2}} = 0
    \]
    and
    \[
        \sqrt{\EE{\left(\bar{a}_t^i - 1\right)^2 \norm{g_i(\theta_t)}^2}} = \frac{\sqrt{2 \mathcal{C}_t} \gamma_t L}{N C_t}.
    \]
    Then if $\ds C_t = \frac{\sqrt{2 \mathcal{C}_t} \gamma_t L}{N \varepsilon}$, for both roots we have $\sqrt{\EE{\left(\bar{a}_t^i - 1\right)^2 \norm{g_i(\theta_t)}^2}} \leq \varepsilon$. Now observe that the quadratic is positive only between the two roots, i.e.\ when $0 \leq \sqrt{\EE{\left(\bar{a}_t^i - 1\right)^2 \norm{g_i(\theta_t)}^2}} \leq \varepsilon$. Therefore,
    \begin{align*}
        \EE{R_i(\theta'_{T}) - p_t^i(\vec{m}'_t) - \left(R_i(\theta_{T}) - p_t^i(\vec{m}_t)\right)}
        &\leq   \frac{\sqrt{2 \mathcal{C}_t} \gamma_t L}{N} \sqrt{\EE{\left(\bar{a}_t^i - 1\right)^2 \norm{g_i(\theta_t)}^2}} \\
            &\quad - C_t \EE{(\bar{a}_t^i - 1)^2 \norm{g_i(\theta_t)}^2} \\
        &\leq   \frac{\sqrt{2 \mathcal{C}_t} \gamma_t L}{N} \sqrt{\EE{\left(\bar{a}_t^i - 1\right)^2 \norm{g_i(\theta_t)}^2}} \\
        &\leq   \frac{\sqrt{2 \mathcal{C}_t}\gamma_t L \varepsilon}{N}.
    \end{align*}
\end{proof}

\begin{claim}[Total bound on utility difference between trajectories]\label{claim:bound-on-reward-loss-per-turn}
    Suppose $\theta'$ denotes some arbitrary reporting strategy and $\theta$ denotes truthful reporting. That is we don't require them to be the same up to the last step. Then:
    \[
        \EE{R_i(\theta'_{T + 1}) - \sum_{t = 1}^T p_t^i(\vec{m}'_t) - \left(R_i(\theta_{T + 1}) - \sum_{t = 1}^T p_t^i(\vec{m}_t)\right)} \leq \frac{\sqrt{2}LG \varepsilon}{N},
    \]
    where $\ds G = \sum_{t = 1}^T \gamma_t \sqrt{\mathcal{C}_t}$.
\end{claim}
\begin{proof}
    Combine Claim~\ref{claim:bound-reward-total-at-time-T-plus-1} with Proposition~\ref{proposition:bound-utility-total-at-time-T-plus-1} and telescope. In particular, starting from the last time step we turn the trajectory produced by possible misreporting to one produced by truthful reporting by comparing the maximal gain in utility from the action of client $i$ at time $t$.
\end{proof}
\section{Proof of Theorem~\ref{theorem:bound-on-payments}}\label{appendix:payment}
We repeat the statement for completeness.

\begin{claim}[Bound on individual payments]
    Suppose all participants are approximately truthful at each time step. Then the total payment is bounded by
    \begin{align*}
        \sum_{t = 1}^T p_t^i(\vec{m}_t)
        &\leq   \frac{\sqrt{2}L G}{N} \left[2\varepsilon^2 + 2 \varepsilon\sigma + 2\zeta^2 + \rho^2\right] + \frac{\sqrt{8}L\varepsilon}{N} \sum_{t = 1}^T \gamma_t \sqrt{\mathcal{C}_t} \norm{\nabla F_i(\theta_t)},
    \end{align*}
    where $\ds G = \sum_{t = 1}^T \gamma_t \sqrt{\mathcal{C}_t}$.
\end{claim}
\begin{proof}
    First, we consider a single time step:\footnote{We divide by $C_t$ to make the exposition more concise.}
    \begin{align*}
        \frac{p_t^i(\vec{m}_t)}{C_t}
        &=      \EE{\norm{\bar{a}_t^i g_i(\theta_t) + b^i_t \xi^i_t}^2 - \frac{1}{N - 1} \sum_{j \not= i} \norm{\bar{a}_t^j g_j(\theta_t) + b^j_t \xi^j_t}^2}\\
        &\leq   \EE{\norm{\bar{a}_t^i g_i(\theta_t)}^2 - \frac{1}{N - 1} \sum_{j \not= i} \norm{\bar{a}_t^j g_j(\theta_t)}^2} + \varepsilon^2\\
        &=      \EE{\norm{\bar{a}_t^i e_i(\theta_t)}^2} + \norm{\bar{a}_t^i \nabla F_i(\theta_t)}^2 - \frac{1}{N - 1} \left(\sum_{j \not= i} \EE{\norm{\bar{a}_t^j e_j(\theta_t)}^2} + \norm{\bar{a}_t^j \nabla F_j(\theta_t)}^2 \right) + \varepsilon^2\\
        &\leq   \EE{\left(\left(\bar{a}_t^i\right)^2 - 1\right) \norm{g_i(\theta_t)}^2} + 2\zeta^2 + \rho^2 + \varepsilon^2\\
        &=      \EE{\left(\bar{a}_t^i - 1\right)^2 \norm{g_i(\theta_t)}^2} + \EE{2\left(\bar{a}_t^i - 1\right) \norm{g_i(\theta_t)}^2} + 2\zeta^2 + \rho^2 + \varepsilon^2 \\
        &\leq   2\varepsilon^2 + 2 \varepsilon \sqrt{\EE{\norm{g_i(\theta_t)}^2}} + 2\zeta^2 + \rho^2 \\
        &\leq   2\varepsilon^2 + 2 \varepsilon \norm{\nabla F_i(\theta_t)} + 2 \varepsilon \sigma + 2\zeta^2 + \rho^2
    \end{align*}
    The second line uses the fact that $\xi_i^t$ is mean zero, variance one and independent of everything, alongside the assumption that all clients are approximately truthful. The third line is rearrangement. The forth line applies Assumption~\ref{assumption:gradient-difference} and~\ref{assumption:variance-difference}. The fifth line is rearrangement. The sixth applies the approximate truthfulness assumption $\EE{\norm{a_t^i g_i(\theta_t) - g_i(\theta_t)}^2} \leq \varepsilon^2$. Claim~\ref{claim:gradient-bound} yields the last inequality.

    Therefore, over all time steps we have:
    \begin{align*}
        \sum_{t = 1}^T p_t^i(\vec{m}_t)
        &\leq   \frac{\sqrt{2}L G}{N} \left[2\varepsilon^2 + 2 \varepsilon\sigma + 2\zeta^2 + \rho^2\right] + \frac{\sqrt{8}L\varepsilon}{N} \sum_{t = 1}^T \gamma_t \sqrt{\mathcal{C}_t} \norm{\nabla F_i(\theta_t)}.
    \end{align*}
\end{proof}
\section{Proof of Theorem~\ref{theorem:convergence-of-sgd}}\label{appendix:convergence}

What follows is the proof of Theorem~\ref{theorem:convergence-of-sgd}, which establishes the convergence rate in the approximately truthful setting. First, we mention a useful result due to~\cite{bottou2018optimization}, then we bound the variance of the aggregate gradient at each step (Claim~\ref{claim:bound-on-variance}). The full proof is at the end. We also give a slightly different bound using a result due to~\cite{chung1954stochastic}.

\subsection{Results from the literature}

\begin{lemma}[Equation 4.23, Theorem 4.7~\citep{bottou2018optimization}]\label{lemma:bottou}
    Let $F$ be a continuously differentiable function that is $H$-smooth and $m$-strongly-convex. Let $g(\theta)$ be a stochastic gradient of $F$ at $\theta$, such that $\EE{g(\theta)} = \nabla F(\theta)$. Suppose there exist scalars $M, M_V \geq 0$, such that $\Var{g(\theta_t)} \leq M + M_V \norm{\nabla F_i(\theta_t)}^2$ for every $t$. If we run SGD with $\gamma_t = \frac{\gamma}{\eta + t}$, where $\gamma > \frac{1}{m}$, $\eta > 0$ and $\gamma_1 \leq \frac{1}{H (M_V + 1)}$, then\footnote{In the original results, there is two additional variables $\mu$ and $\mu_G$. Here we can set them to $\mu_G = \mu = 1$, so we choose to simplify the write-up and ignore them.}
    \begin{equation}
        \EE{F(\theta_{t + 1}) - F(\theta^*)} \leq (1 - \gamma_tm) \EE{F(\theta_t) - F(\theta^*)} + \frac{\gamma_t^2HM}{2},
    \end{equation}
    and
    \begin{equation}
        \EE{F(\theta_{t}) - F(\theta^*)} \leq \max\left\{\frac{\gamma^2 H M}{2(\gamma m - 1)(\eta + t)}, \frac{(\eta + 1) (F(\theta_1) - F(\theta^*))}{\eta + t}\right\}.
    \end{equation}
\end{lemma}

\begin{lemma}[Lemma 1~\citep{chung1954stochastic}]\label{lemma:chung}
    Suppose that $\{b_n\}_{n \in \NN}$ is a sequence of real numbers such that for $n \geq n_0$,
    \[
        b_{n + 1} \leq \left(1 - \frac{c}{n}\right) b_n + \frac{c_1}{n^{p + 1}}
    \]
    where $c > p > 0$, $c_1 > 0$. Then
    \[
        b_n \leq \frac{c_1}{c - p} \frac{1}{n^p} + O\left(\frac{1}{n^{p + 1}} + \frac{1}{n^c}\right).
    \]
\end{lemma}

\subsection{The proof}

\begin{claim}[Bound on the variance of the aggregated gradient]\label{claim:bound-on-variance}
    Suppose that there exist scalars $M, M_V \geq 0$, such that for every $t$ we have: \[\EE{\norm{e_i(\theta_t)}^2} \leq M + M_V \norm{\nabla F(\theta_t)}^2.\] If all participants are approximately truthful at each time step $t$, then:
    \begin{align*}
        \Var{\frac{1}{N} \sum_{i = 1}^n a_t^i g_i(\theta_t) + b^i_t \xi^i_t}
        &\leq   \frac{2 \left(2\varepsilon^2 + M + 2 M_V \zeta^2\right)}{N} + \frac{2 M_V}{N}\norm{\nabla F(\theta_t)}^2,
    \end{align*}
    where for a stochastic gradient $g$ we use $\Var{g} = \EE{\norm{g - \EE{g}}^2}$ for the variance of the squared $\ell_2$ norm of the error of $g$ with respect to its mean $\EE{g}$.
\end{claim}
\begin{proof}
    Observe the following series of inequalities. Recall that $\xi^i_t$ is mean zero and variance one and is independent of everything, so we can quite easily take care of the noise terms.
    \begin{align*}
        \Var{\frac{1}{N} \sum_{i = 1}^N a_t^i g_i(\theta_t)}
        &=      \EE{\norm{\frac{1}{N} \sum_{i = 1}^N a_t^i g_i(\theta_t) + b^i_t \xi^i_t}^2} - \norm{\EE{\frac{1}{N} \sum_{i = 1}^N a_t^i g_i(\theta_t) + b^i_t \xi^i_t}}^2 \\
        &=      \frac{1}{N^2} \EE{\norm{\sum_{i = 1}^N a_t^i e_i(\theta_t) + \sum_{i = 1}^N a_t^i \nabla F_i(\theta_t)}^2} + \frac{\varepsilon^2}{N} - \frac{1}{N^2} \norm{\sum_{i = 1}^N a_t^i \nabla F_i(\theta_t)}^2 \\
        &=      \frac{1}{N^2} \sum_{i = 1}^N \EE{\norm{a_t^i e_i(\theta_t)}^2} + \frac{1}{N^2} \norm{\sum_{i = 1}^N a_t^i \nabla F_i(\theta_t)}^2 - \frac{1}{N^2} \norm{\sum_{i = 1}^N a_t^i \nabla F_i(\theta_t)}^2 + \frac{\varepsilon^2}{N} \\
        &\leq   \frac{3\varepsilon^2}{N} + \frac{2}{N^2} \sum_{i = 1}^N \EE{\norm{e_i(\theta_t)}^2} \\
        &\leq   \frac{2\left(2\varepsilon^2 + M\right)}{N} + \frac{2 M_V}{N^2} \sum_{i = 1}^N \norm{\nabla F_i(\theta_t)}^2 \\
        &\leq   \frac{2 \left(2\varepsilon^2 + M + M_V \zeta^2\right)}{N} + \frac{2 M_V}{N}\norm{\nabla F(\theta_t)}^2
    \end{align*}
\end{proof}

\begin{proof}[Proof of Theorem~\ref{theorem:convergence-of-sgd}]
    Once we have Claim~\ref{claim:bound-on-variance} under out belt, we invoke Lemma~\ref{lemma:bottou} with $M = \frac{2(2\varepsilon^2 + M + M_V \zeta^2)}{N}$, $M_V = \frac{2 M_V}{N}$, and $\gamma_t = \frac{\gamma}{\eta + t} = \frac{4}{m(\eta + t)}$, where $\eta = \frac{4H (2 M_V + N)}{mN}$. This yields:
    \begin{align*}
        \EE{F(\theta_{t}) - F(\theta^*)} 
        &\leq   \max\left\{\frac{16 H (2\varepsilon^2 + M + M_V \zeta^2)}{3 N m^2(\eta + t)}, \frac{(\eta + 1)(F(\theta_1) - F(\theta^*))}{\eta + t} \right\}.
    \end{align*}
\end{proof}

\begin{proof}[Alternative bound for Theorem~\ref{theorem:convergence-of-sgd}]
    Once we have Claim~\ref{claim:bound-on-variance} under out belt, we invoke Lemma~\ref{lemma:bottou} with $M = \frac{2(2\varepsilon^2 + M + M_V \zeta^2)}{N}$, $M_V = \frac{2 M_V}{N}$, and $\gamma_t = \frac{\gamma}{\eta + t} = \frac{4}{m(\eta + t)}$, where $\eta = \frac{4H (2 M_V + N)}{mN}$. This yields:
    \begin{align*}
        \EE{F(\theta_{t + 1}) - F(\theta^*)}
        &\leq \left(1 - \frac{4}{\eta + t}\right) \EE{F(\theta_t) - F(\theta^*)} + \frac{16 H (2\varepsilon^2 + M + M_V \zeta^2)}{N m^2 (\eta + t)^2}.
    \end{align*}

    Finally, we use Lemma~\ref{lemma:chung} with $p = 1$, $c = 4$ and $c_1 = \frac{16 H (\varepsilon^2 + M)}{N m^2}$, to get:
    \begin{align*}
        \EE{F(\theta_t) - F(\theta^*)} \leq \frac{16 H (2\varepsilon^2 + M + M_V \zeta^2)}{3N m^2(\eta + t)} + O\left(\frac{1}{t^2} + \frac{1}{t^4}\right).
    \end{align*}
\end{proof}
\section{Extending the action space}\label{appendix:extending-action-space}
The generalizations of our main results follow from the main analysis with very few modifications. The proofs below state what changes need to be made and note any peculiarities that emerge from generalizing the analysis.

\begin{proof}[Proof of Claim~\ref{claim:mixed-strategy}]
    In all equations in Appendices~\ref{appendix:approximate-truthfulness},~\ref{appendix:payment} and~\ref{appendix:convergence} the scaling factor $a^i_t$ always appears under the expectation sign. We need to make sure that in all proofs we substitute $\abs{\bar{b}^i_t}$ and $(\bar{b}^i_t)^2$ with $\sqrt{\EE{(\bar{b}^i_t)^2}}$ and $\EE{(\bar{b}^i_t)^2}$. Then the results follow quite directly for mixed strategies that are independent of the stochasticity in the learning protocol.
\end{proof}

\begin{proof}[Proof of Claim~\ref{claim:history-dependent-strategy}]
We can allow $\bar{a}^i_t$ to depend on the history of the protocol, that is $\bar{a}^i_t$ can be a random variable bounded below by 1 almost surely that depends on the previous model parameters $\{\theta_j\}_{j = 1}^t$ and the previous gradients $\{g_i(\theta_1)\}_{j = 1}^{t - 1}$ sent by client $i$. In particular, the important fact for the proofs above is that $\bar{a}^i_t$ is independent of the randomness in $g_i(\theta_t)$ and that $\bar{a}^i_t \geq 1$ almost surely. The same goes for the DP noise $\bar{b}^i_t$ as long as it's independent of $g_i(\theta_t)$ and $\bar{a}^i_t$. In particular, notice that in the results in Claim~\ref{claim:bound-on-payment-per-turn-hold-prev-fixed} and Claim~\ref{claim:bound-reward-total-at-time-T-plus-1} the expectation on the right-hand side is taken just on the random variable $g_i(\theta_t)$ (which depends on the history of the protocol), so these bounds hold in expectation over the gradient that is sampled at time step $t$.
\end{proof}

\begin{proof}[Proof of Claim~\ref{claim:allowing-angles}]
    Recall that $\norm{h^i_t} = a^t_i \norm{g_i(\theta_t)}$ and $\inner{h^i_t}{g_i(\theta_t)} \geq \norm{g_i(\theta_t)}^2$ by the construction in Claim~\ref{claim:allowing-angles}. Moreover, Cauchy-Schwartz gives $\inner{h^i_t}{g_i(\theta_t)} \leq a^i_t \norm{g_i(\theta_t)}^2$, so
    \[
        \norm{g_i(\theta_t)}^2 \leq \inner{h^i_t}{g_i(\theta_t)} \leq a^i_t \norm{g_i(\theta_t)}^2.
    \]
    Now throughout our analysis in Appendices~\ref{appendix:approximate-truthfulness},~\ref{appendix:payment} and~\ref{appendix:convergence} we need to substitute $\bar{a}^i_t g_i(\theta_t)$ with $h^i_t$. Note that above we used $a^t_i$ to denote the strategy used under the statement of Claim~\ref{claim:allowing-angles}, while here $\bar{a}^i_t$ is the scaling factor used in our analysis in Appendix~\ref{appendix:approximate-truthfulness}. We only need to make sure that in Claim~\ref{claim:bound-on-payment-per-turn-hold-prev-fixed} the right-hand side remains as $C_t \EE{\norm{g_i(\theta_t)}^2 - \norm{h^i_t}}$, without bounding that quantity with something equivalent to $- C_t (\bar{a}_t^i - 1)^2 \EE{\norm{g_i(\theta_t)}^2}$.
\end{proof}

\section{Bound on stochastic gradient norm}\label{appendix:gradient-norm}

\begin{claim}\label{claim:gradient-bound}
    Let $F$ and $g$ satisfy the same conditions as $F_i$ and $g_i$ from Section~\ref{subsection:fl-setup}. Let $\theta^* \in \Theta$ be a minimizer of $F$. For any $\theta \in \Theta$ we have
    \[
        \EE{\norm{g_i(\theta)}} \leq \sqrt{\EE{\norm{g_i(\theta)}^2}} \leq \sqrt{\EE{\norm{g_i(\theta)}^2}} \leq \norm{\nabla F_i(\theta)} + \sigma.
    \]
\end{claim}
\begin{proof}
    Let $e_i(\theta) = g_i(\theta) - \nabla F_i(\theta)$ be the gradient noise. Observe:
    \begin{align*}
        \EE{\norm{g_i(\theta)}^2}
        &=  \EE{\norm{\nabla F_i(\theta) + e_i(\theta)}^2} \\
        &=  \EE{\norm{\nabla F_i(\theta)}^2} + 2\EE{\inner{\nabla F_i(\theta)}{e_i(\theta)}} + \EE{\norm{e_i(\theta)}^2} \\
        &=  \norm{\nabla F_i(\theta)}^2 + \EE{\norm{e_i(\theta)}^2} \\
        &\leq    \norm{\nabla F_i(\theta)}^2 + \sigma^2
    \end{align*}
    The third line uses the fact that $\EE{e_i(\theta)} = 0$ for any fixed $\theta \in \Theta$. Finally, because $\sqrt{\cdot}$ is concave and $\norm{g_i(\theta)}^2$ is a non-negative random variable, Jensen's inequality finishes the claim.
\end{proof}
\newpage
\section{Additional experimental results}\label{appendix:experiments}

\subsection{Technical details for experiments}

The LEAF datasets were downloaded using \url{https://github.com/TalwalkarLab/leaf/}~\citep{caldas2018leaf}. Each experiment run was performed on a single GPU with 48G of memory.

\subsection{Results with FedSGD with coordinate-wise median aggregation}

\begin{figure*}[htbp]
    \centering
    \subfloat[FeMNIST dataset]
    {
        \includegraphics[width=0.3\textwidth, trim=0mm 0mm 15mm 14mm, clip]{./images/femnist-b1=0.0-ls=1_median.eps}
        \label{figure:fedsgd-median-all-femnist}
    }
    \hfill
    \subfloat[Shakespeare dataset]
    {
        \includegraphics[width=0.3\textwidth, trim=0mm 0mm 15mm 14mm, clip]{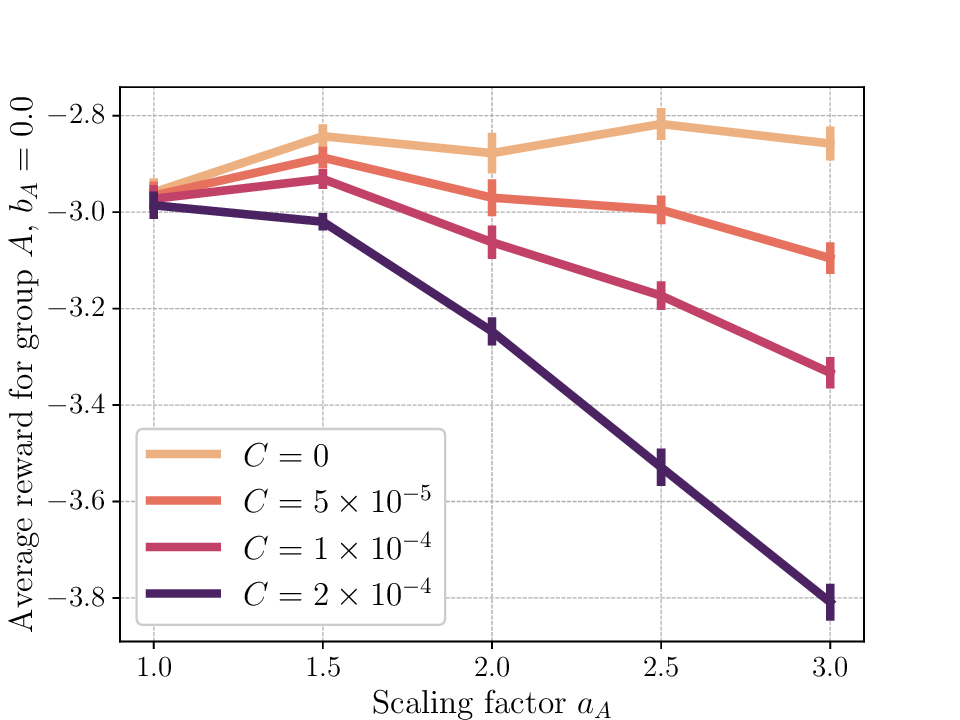}
        \label{figure:fedsgd-median-all-shakespear}
    }
    \hfill
    \subfloat[Twitter/Sent140 dataset]
    {
        \includegraphics[width=0.3\textwidth, trim=0mm 0mm 15mm 14mm, clip]{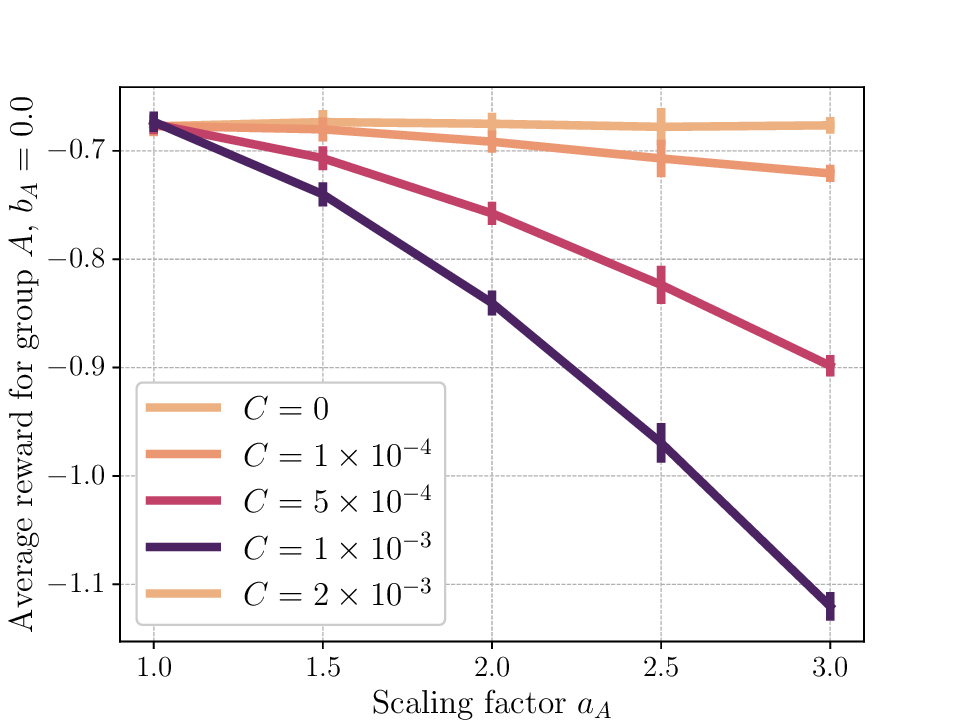}
        \label{figure:fedsgd-median-all-twitter}
    }
    \caption{Experiments with FedSGD with coordinate-wise median aggregation. The $x$ axis is the scaling factor $a_A$, and the $y$ axis is the utility for the client who misreports. Noise level is set to zero. All other clients are truthful.}
    \label{figure:fedsgd-median-all}
\end{figure*}

\subsection{Results with FedAvg}

\begin{figure*}[htbp]
    \centering
    \subfloat[FeMNIST dataset]
    {
        \includegraphics[width=0.3\textwidth, trim=0mm 0mm 15mm 14mm, clip]{./images/femnist-b1=0.0-ls=3.eps}
        \label{figure:fedavg-all-femnist}
    }
    \hfill
    \subfloat[Shakespeare dataset]
    {
        \includegraphics[width=0.3\textwidth, trim=0mm 0mm 15mm 14mm, clip]{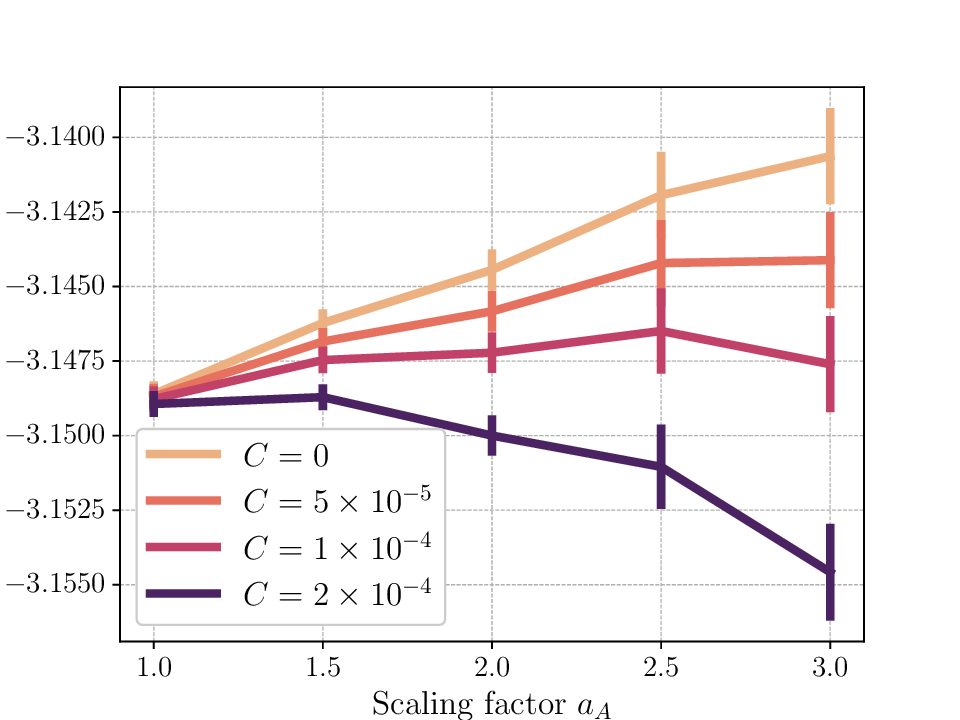}
        \label{figure:fedavg-all-shakespear}
    }
    \hfill
    \subfloat[Twitter/Sent140 dataset]
    {
        \includegraphics[width=0.3\textwidth, trim=0mm 0mm 15mm 14mm, clip]{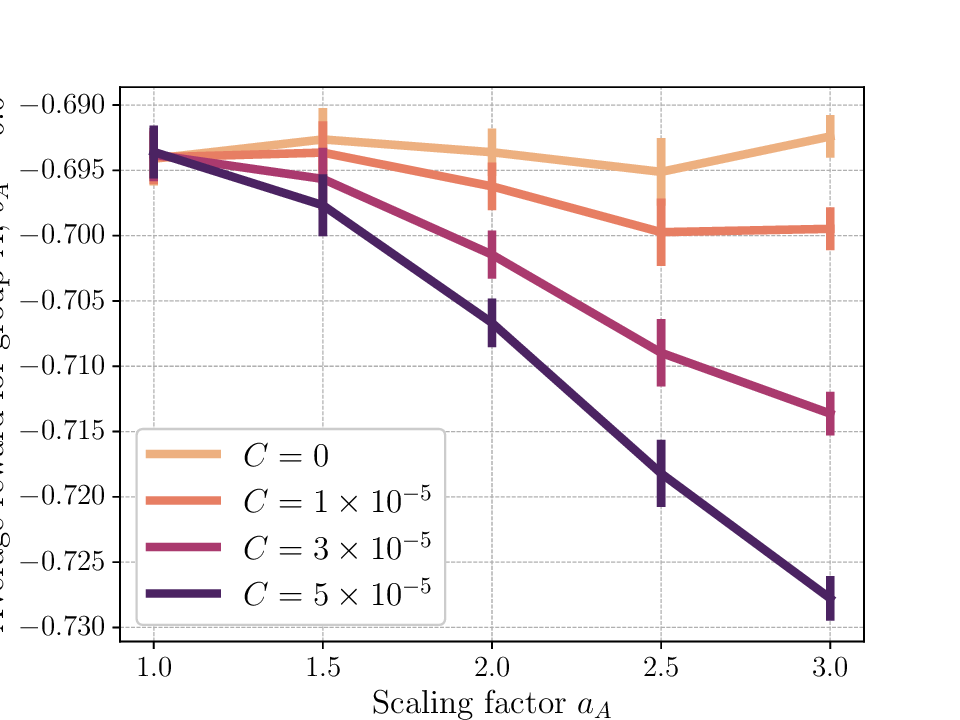}
        \label{figure:fedavg-all-twitter}
    }
    \caption{Experiments with FedAvg. The $x$ axis is the scaling factor $a_A$, and the $y$ axis is the utility for the misreporting client in group $A$. Noise level is set to zero. All other clients are truthful.}
    \label{figure:fedavg-all}
\end{figure*}

\newpage
\subsection{Results with FedSGD with different noise levels}

\begin{figure*}[!htbp]
    \centering
    \subfloat[$a_A = 1.0$]
    {
        \includegraphics[width=0.3\textwidth, trim=0mm 0mm 15mm 14mm, clip]{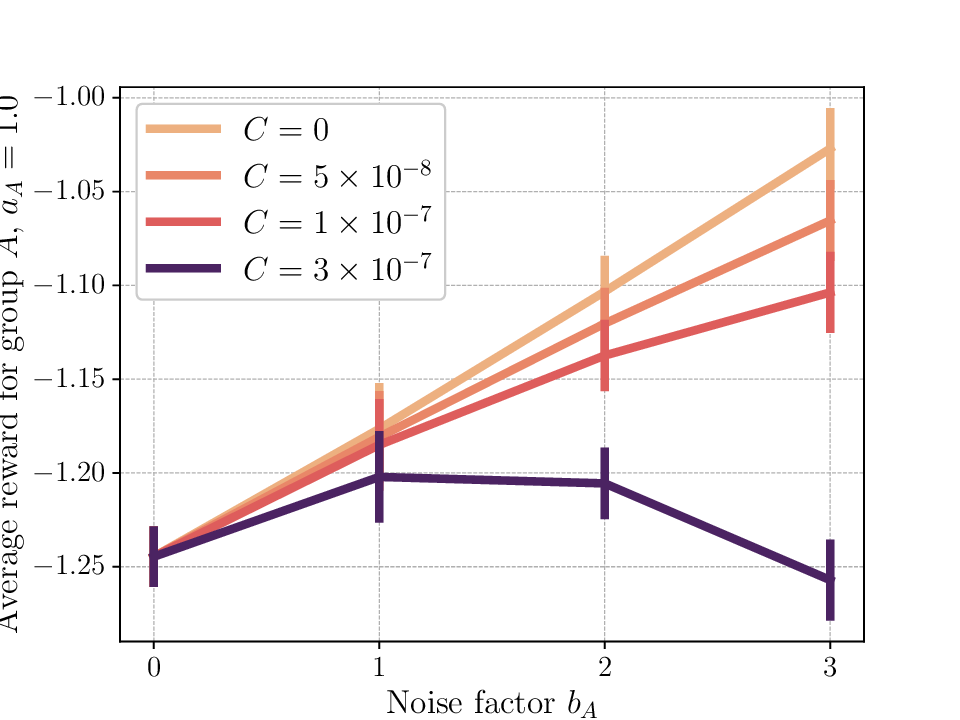}
        \label{figure:fedsgd-femnist-noise-1}
    }
    \hfill
    \subfloat[$a_A = 2.0$]
    {
        \includegraphics[width=0.3\textwidth, trim=0mm 0mm 15mm 14mm, clip]{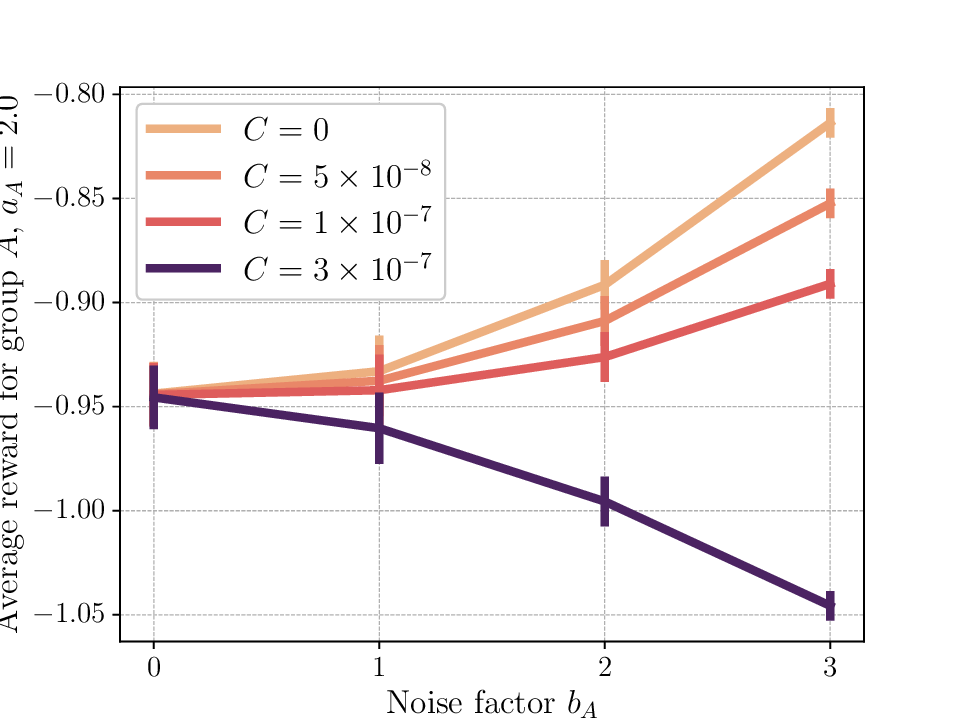}
        \label{figure:fedsgd-femnist-noise-2}
    }
    \hfill
    \subfloat[$a_A = 3.0$]
    {
        \includegraphics[width=0.3\textwidth, trim=0mm 0mm 15mm 14mm, clip]{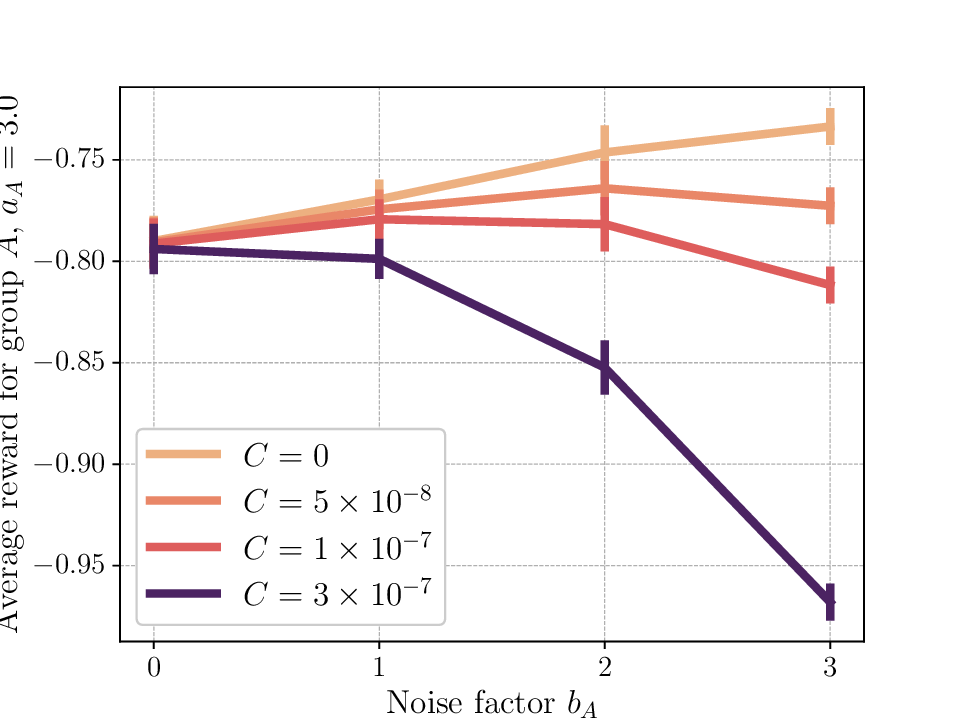}
        \label{figure:fedsgd-femnist-noise-3}
    }
    \caption{FedSGD experiments with FeMNIST dataset with varying noise levels.}
    \label{figure:fedsgd-femnist-noise}
\end{figure*}

\begin{figure*}[!htbp]
    \centering
    \subfloat[$a_A = 1.0$]
    {
        \includegraphics[width=0.3\textwidth, trim=0mm 0mm 15mm 14mm, clip]{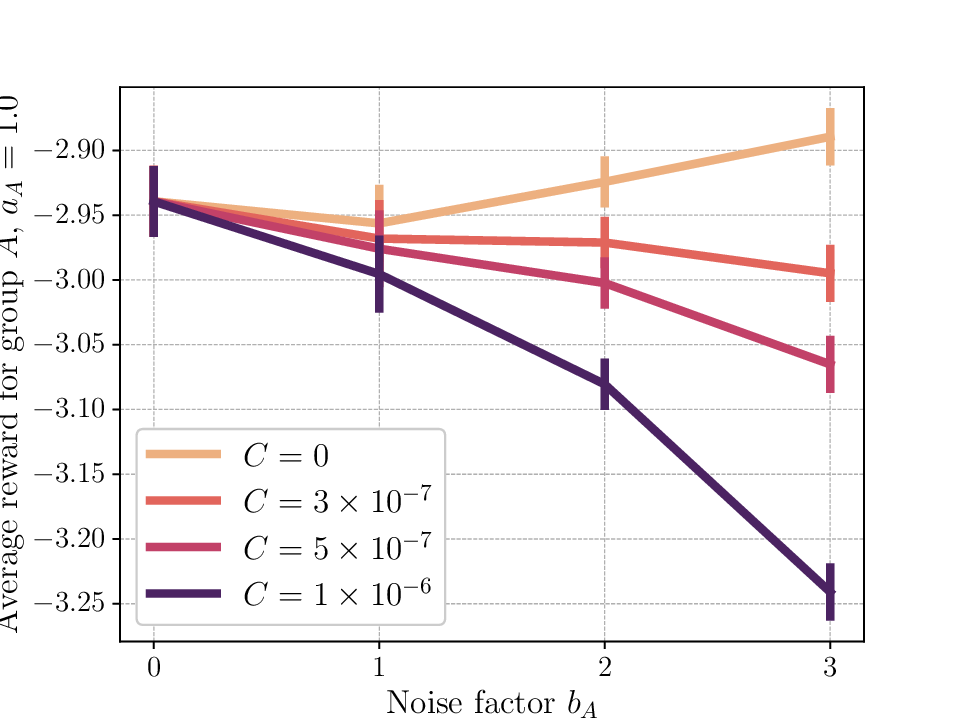}
        \label{figure:fedsgd-shakespear-noise-1}
    }
    \hfill
    \subfloat[$a_A = 2.0$]
    {
        \includegraphics[width=0.3\textwidth, trim=0mm 0mm 15mm 14mm, clip]{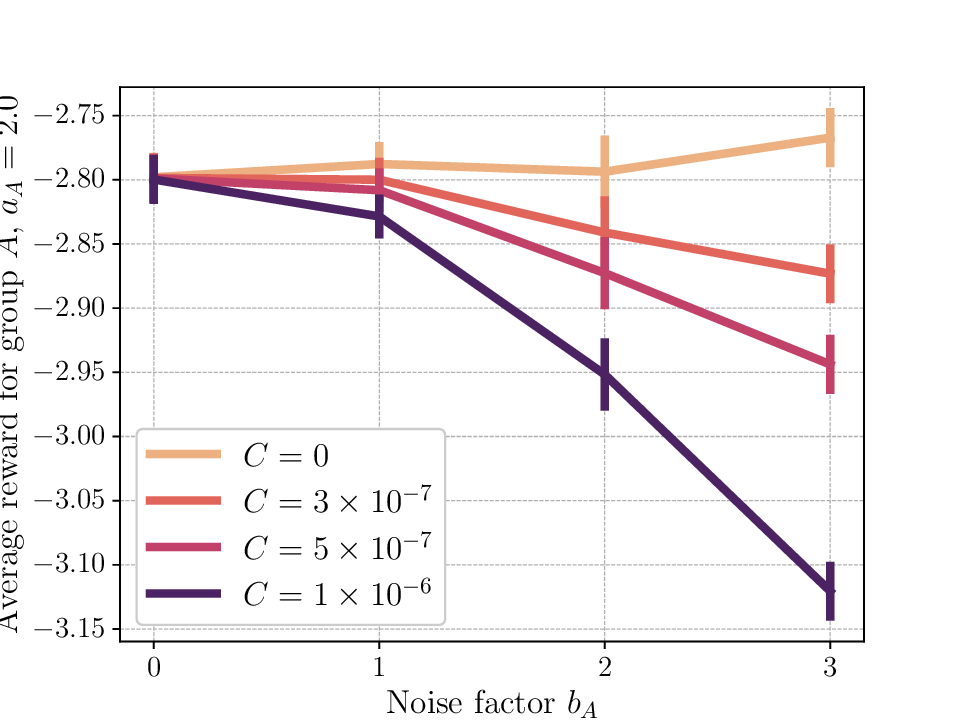}
        \label{figure:fedsgd-shakespear-noise-2}
    }
    \hfill
    \subfloat[$a_A = 3.0$]
    {
        \includegraphics[width=0.3\textwidth, trim=0mm 0mm 15mm 14mm, clip]{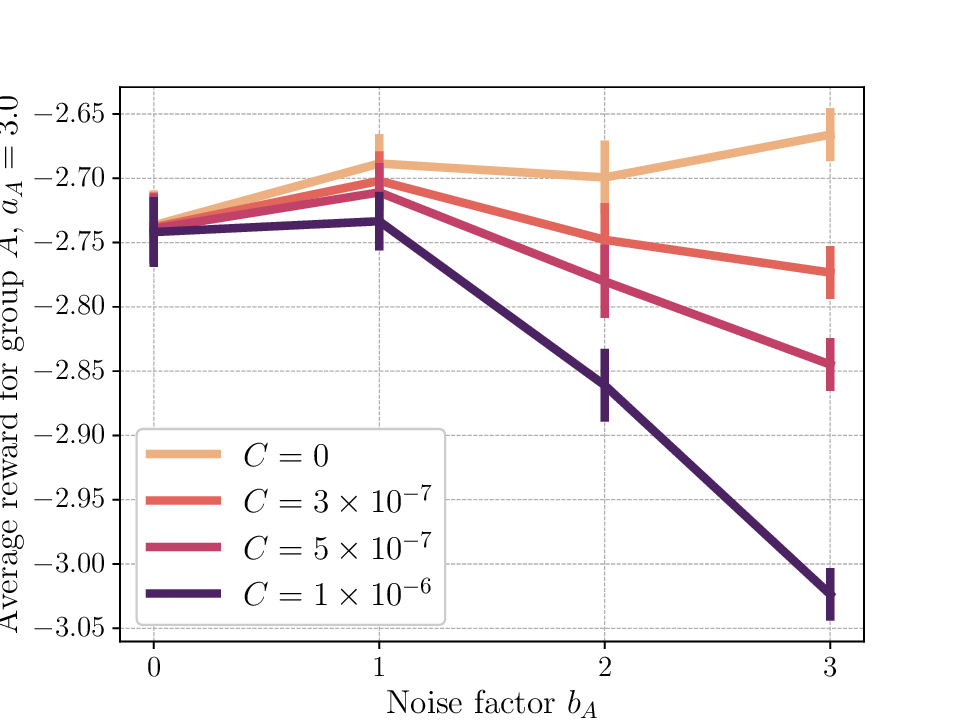}
        \label{figure:fedsgd-shakespear-noise-3}
    }
    \caption{FedSGD experiments with Shakespeare dataset with varying noise levels.}
    \label{figure:fedsgd-shakespear-noise}
\end{figure*}

\begin{figure*}[!htbp]
    \centering
    \subfloat[$a_A = 1.0$]
    {
        \includegraphics[width=0.3\textwidth, trim=0mm 0mm 15mm 14mm, clip]{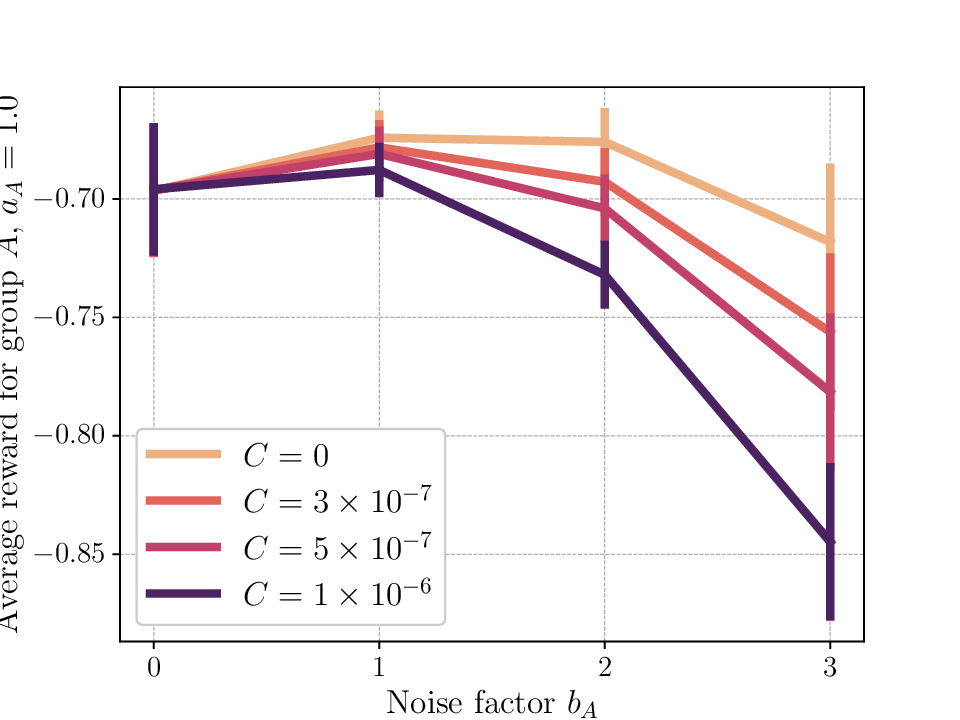}
        \label{figure:fedsgd-twitter-noise-1}
    }
    \hfill
    \subfloat[$a_A = 2.0$]
    {
        \includegraphics[width=0.3\textwidth, trim=0mm 0mm 15mm 14mm, clip]{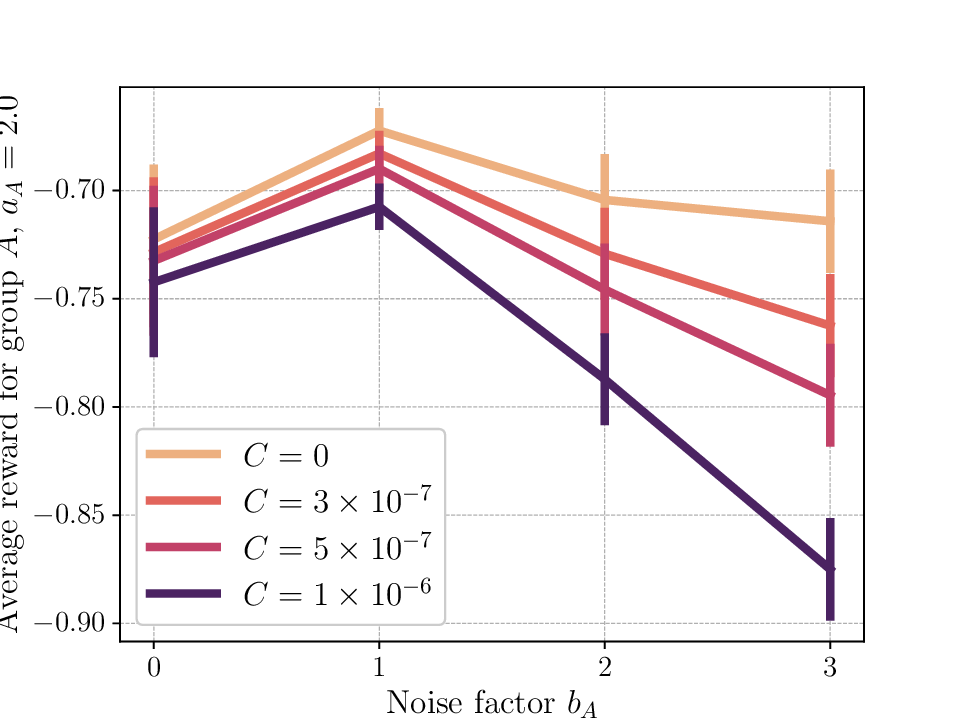}
        \label{figure:fedsgd-twitter-noise-2}
    }
    \hfill
    \subfloat[$a_A = 3.0$]
    {
        \includegraphics[width=0.3\textwidth, trim=0mm 0mm 15mm 14mm, clip]{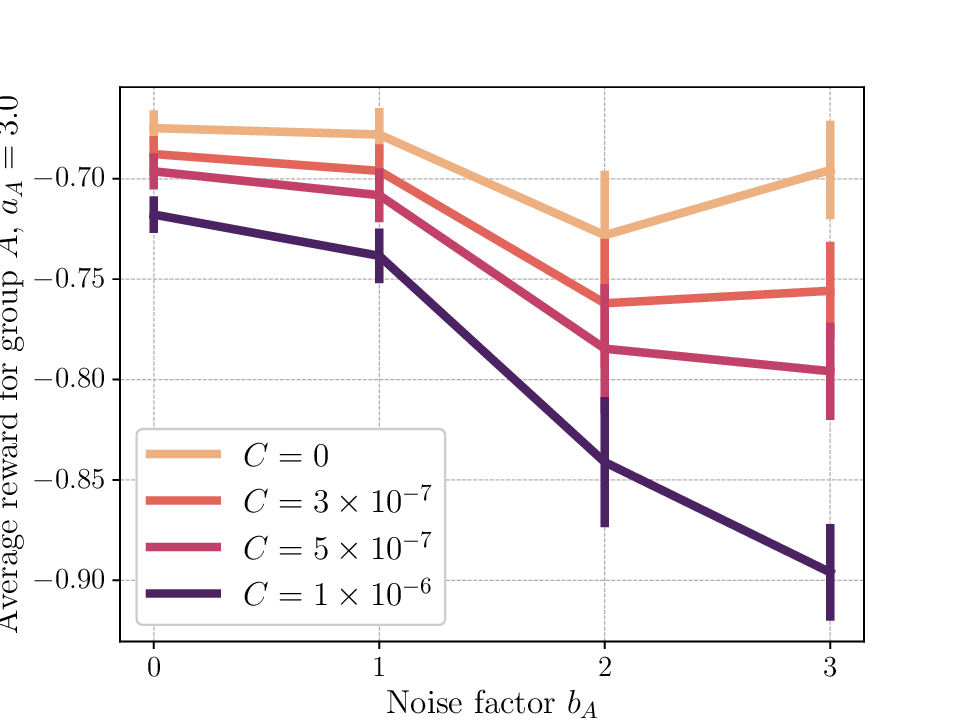}
        \label{figure:fedsgd-twitter-noise-3}
    }
    \caption{FedSGD experiments with Twitter dataset with varying noise levels.}
    \label{figure:fedsgd-twitter-noise}
\end{figure*}

\end{document}